\documentclass{article}

\usepackage{microtype}
\usepackage{graphicx}
\usepackage{booktabs} %

\usepackage{amssymb,amsmath,amsthm}
\usepackage{ifxetex,ifluatex}
\usepackage{fixltx2e} %
\usepackage{bbm}
\usepackage{xifthen}
\usepackage{todonotes}
\usepackage{xcolor}
\usepackage{caption}
\usepackage{subcaption}
\usepackage{macros}

\RequirePackage{fancyhdr}
\RequirePackage{color}
\RequirePackage{algorithm}
\RequirePackage{algorithmic}
\RequirePackage{natbib}

\renewcommand{\cite}[1]{\citep{#1}}

\def\shownotes{1}  \ifnum\shownotes=1
\newcommand{\authnote}[2]{[ #1: #2]}
\else
\newcommand{\authnote}[2]{}
\fi

\usepackage{hyperref}
\definecolor{mydarkblue}{rgb}{0,0.08,0.45}
\hypersetup{ %
pdftitle={},
pdfauthor={},
pdfsubject={},
pdfkeywords={},
pdfborder=0 0 0,
pdfpagemode=UseNone,
colorlinks=true,
linkcolor=mydarkblue,
citecolor=mydarkblue,
filecolor=mydarkblue,
urlcolor=mydarkblue,
pdfview=FitH}

\usepackage[letterpaper]{geometry}
\usepackage{authblk} %

\title{\vspace{-1.0cm}Optimal Regularization Can Mitigate Double Descent}
\renewcommand\footnotemark{}
\author[1]{Preetum Nakkiran\thanks{Emails: \texttt{preetum@cs.harvard.edu}, \texttt{pvenkat@g.harvard.edu}, \texttt{sham@cs.washington.edu}, \texttt{tengyuma@stanford.edu}}}
\author[1]{Prayaag Venkat}
\author[2]{Sham Kakade}
\author[3]{Tengyu Ma}
\affil[1]{Harvard University}
\affil[2]{Microsoft Research \& University of Washington}
\affil[3]{Stanford University}

\date{}

\begin{document}

\maketitle

\newcommand{\hbeta}{\hat{\beta}}
\newcommand{\bR}{\bar{R}}
\newcommand{\B}{\beta^*}
\newcommand{\Ln}[1]{\lambda^{\textup{opt}}_{#1}}
\newcommand{\Bn}[1]{\hbeta^{\textup{opt}}_{#1}}
\renewcommand{\t}{\widetilde}
\begin{abstract}
Recent empirical and theoretical studies have shown that many learning algorithms -- from linear regression to neural networks -- can have test performance that is non-monotonic in quantities such the sample size and model size. This striking phenomenon, often referred to as ``double descent'', has raised questions of if we need to re-think our current understanding of generalization. In this work, we study whether the double-descent phenomenon can be avoided by using optimal regularization. Theoretically, we prove that for certain linear regression models with isotropic data distribution, optimally-tuned $\ell_2$ regularization achieves monotonic test performance as we grow either the sample size or the model size.
We also demonstrate empirically that optimally-tuned $\ell_2$ regularization can mitigate double descent for more general models, including neural networks.
Our results suggest that it may also be informative to study the
test risk scalings of various algorithms in the context of appropriately tuned regularization.
\end{abstract}

\section{Introduction}
Recent works have demonstrated a ubiquitous ``double descent''
phenomenon present in a range of machine learning models, including  decision trees, random features, linear regression,
and deep neural networks
\cite{opper1995statistical, opper2001learning, advani2017high, spigler2018jamming, belkin2018reconciling, geiger2019jamming, nakkiran2020deep, belkin2019two, hastie2019surprises, bartlett2019benign, muthukumar2019harmless, bibas2019new, Mitra2019UnderstandingOP, mei2019generalization, liang2018just, liang2019risk,xu2019number,dereziski2019exact,lampinen2018analytic, deng2019model, nakkiran2019more}.
The phenomenon is that
models exhibit a peak of high test risk when they are
just barely able to fit the train set, that is, to \emph{interpolate}.
For example, as we increase the size of models, test risk 
first decreases, then increases to a peak around when effective model size is close to the training data size, and then decreases again in the overparameterized regime.
Also surprising is that \citet{nakkiran2020deep} observe a double descent as we increase \textit{sample size}, i.e. for a fixed model,  training the model with more data can hurt test performance.

\begin{figure}[th!]
	\centering
	\includegraphics[width=0.6\textwidth]{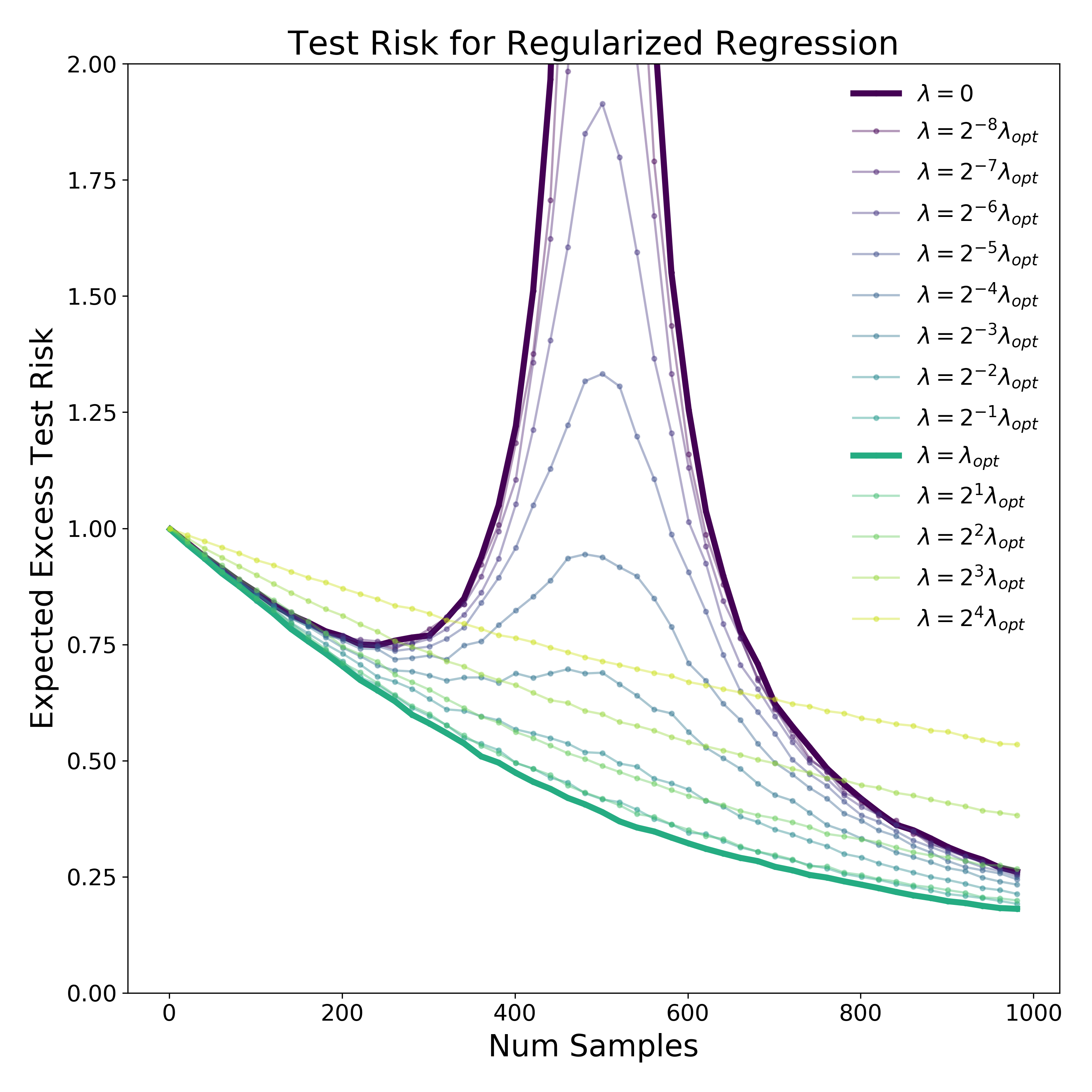}
	\caption{{\bf Test Risk vs. Num. Samples for Isotropic Ridge Regression in $d=500$ dimensions.}
		Unregularized regression is non-monotonic in samples,
		but optimally-regularized regression ($\lambda = \lambda_{opt}$) 
		is monotonic.
		The sample distribution is $(x, y)$ where $x \sim \cN(0, I_d)$
		and $y = \langle \beta^*, x \rangle + \cN(0, \sigma^2)$
		for $d=500, \sigma=0.5$, and $||\beta^*||_2 = 1$.
		For $\lambda > 0$, the ridge estimator on $n$ samples is
		$\hbeta_\lambda := \argmin_{\beta} ||X \beta - \vec{y}||_2^2 + \lambda ||\beta||_2^2$.
		In this setting, the optimal regularizer $\lambda_{opt}$ does not depend
		on number of samples $n$ (Lemma~\ref{lem:optrisk}), but this is not always true -- see Figure~\ref{fig:nonisotropic}.
		}
	\label{fig:isotropic}
\end{figure}

These striking observations highlight a potential gap
in our understanding of generalization 
and an opportunity for improved methods. 
Ideally, we seek to use learning algorithms
which robustly improve performance as the data
or model size grow and do not
exhibit such unexpected non-monotonic behaviors.
In other words, we aim to improve the test performance
in situations which would otherwise exhibit high test risk
due to double descent.
Here, a natural strategy would be to use a regularizer and tune its strength on a validation set. 

This motivates the central question of this work: 

\begin{center}
\textit{When does optimally tuned regularization mitigate or remove the double-descent phenomenon?}
\end{center}

Another motivation to start this line of inquiry is
the observation that the double descent phenomenon is largely observed for \emph{unregularized} or \emph{under-regularized} models in practice. As an example, Figure~\ref{fig:isotropic} shows a simple linear ridge regression setting
in which the unregularized estimator exhibits double descent, but an optimally-tuned regularizer has monotonic test performance.

\paragraph{Our Contributions:} We study this question from both a theoretical and empirical perspective. Theoretically, we start with the setting of high-dimensional linear regression. Linear regression is a sensible starting point to study these questions,
since it already exhibits many of the
qualitative features of double descent in more complex models
(e.g.~\citet{belkin2019two, hastie2019surprises} and further related works in Section~\ref{sec:related}).

This work shows that optimally-tuned ridge regression can achieve both sample-wise monotonicity and model-size-wise monotonicity under certain assumptions.
Concretely, we show
\begin{itemize}
	\item[1.] {\bf Sample-wise monotonicity:}
	In the setting of well-specified linear regression with 
    isotropic features/covariates (Figure~\ref{fig:isotropic}),
	we prove that optimally-tuned ridge regression yields
	monotonic test performance with increasing samples. 
    That is, more data never hurts for optimally-tuned ridge regression (see Theorem~\ref{thm:main1}).
	
	\item[2.] {\bf Model-wise monotonicity:}
    We consider a setting where the input/covariate lives in a high-dimensional ambient space with isotropic covariance.
    Given a fixed model size $d$ (which might be much smaller than ambient dimension), we consider the family of models
    which first project the input to a
    random $d$-dimensional subspace,
    and then compute a linear function in this projected ``feature space.''
    (This is nearly identical to models of double-descent considered in
    ~\citet[Section 5.1]{hastie2019surprises}).
    We prove that in this setting, as we grow the model-size, optimally-tuned ridge regression over the projected features has monotone test performance. That is, with optimal regularization, bigger models are always better or the same. (See Theorem~\ref{thm:main-proj}).
    \item[3.] {\bf Monotonicity in the real-world:} 
    We also demonstrate
several richer empirical settings where optimal $\ell_2$ regularization
induces monotonicity,
including random feature classifiers
and convolutional neural networks.
This suggests that the mitigating effect of optimal regularization
may hold more generally in broad machine learning contexts. (See Section~\ref{sec:experiment}).

\end{itemize}
A few remarks are in order:

\noindent{\bf Problem-specific vs Minimax and Bayesian.}
It is worth noting that our results hold for \textit{all}  linear ground-truths, rather than holding for only the worst-case ground-truth or a random ground-truth. Indeed, the minimax optimal estimator or the Bayes optimal estimator are both trivially sample-wise and model-wise monotonic
\emph{with respect to the minimax risk or the Bayes risk}.
However, they do not guarantee monotonicity of the
risk itself for a given fixed problem.

\noindent{\bf Universal vs Asymptotic.} We also remark that our analysis is not only non-asymptotic but also works for all possible input dimensions, model sizes, and sample sizes. 
Prior works on double descent mostly rely on asymptotic assumptions that send the sample size or the model size to infinity in a specific manner. To our knowledge, the results herein are  the first non-asymptotic sample-wise and model-wise monotonicity results for linear regression. (See discussion of related works
\citet{hastie2019surprises, mei2019generalization} for related results
in the asymptotic setting).

Finally, we note that our claims are about monotonicity of the actual test risk, instead of the monotonicity of the generalization bounds (e.g., results in~\cite{wei2019regularization}). 

\noindent{\bf Towards a more general characterization.} Our theoretical results crucially rely on the covariance of the data being isotropic. A natural next question is if and when the same results can hold more generally. A full answer to this question is beyond the scope of this paper, though we give the following results: 

\begin{enumerate}
\item Optimally-tuned ridge regression is \emph{not} always sample-monotonic: we show a counterexample
for a certain non-Gaussian data distribution and heteroscedastic noise. 
We are not aware of prior work pointing out this fact.
(See Section~\ref{sec:c1} for the counterexample and intuitions.)

\item For non-isotropic Gaussian covariates, we can achieve sample-wise monotonicity with a regularizer that depends on the population covariance matrix of data.
This suggests unlabeled data might also help
mitigate double descent in some settings, because the population covariance can be estimated from unlabeled data. 
(See Section~\ref{sec:general}).

\item For non-isotropic Gaussian covariates,
we conjecture that optimally-tuned ridge regression
is sample-monotonic even with a standard $\ell_2$ regularizer (as in Figure~\ref{fig:nonisotropic}).
We derive a sufficient condition for this conjecture, which we verify numerically on a wide variety of cases.
\end{enumerate}

The last two results above highlight the importance of the form of the regularizer, which leads to the open question:
``How do we design good regularizers which
mitigate or remove double descent?''
We hope that our results can motivate
future work on mitigating  the double descent phenomenon,
and allow us to train high performance models
which do not exhibit nonmonotonic behaviors.

\subsection{Related Works}
\label{sec:related}
The study of nonmonotonicity in learning algorithms existed prior to
double descent and has a long history
going back to (at least) \cite{trunk1979problem} and
\cite{lecun1991second, le1991eigenvalues}, where the former was
largely empirical observations and the latter
studied the sample non-nonmonotonicity of unregularized linear
regression in terms of the eigenspectrum of the covariance matrix; the difference to our works is that we study this in the context of optimal regularization.
 In fact, ~\citet{duin1995small, duin2000classifiers, opper2001learning, loog2012dipping}.
\citet{loog2019minimizers} introduces the same notion of
risk monotonicity which we consider,
and studies several examples of monotonic and non-monotonic procedures.

Double descent of test risk as a function of model size was
considered recently in more generality by \citet{belkin2018reconciling}.
Similar behavior was observed empirically in earlier work in somewhat more restricted settings 
\cite{trunk1979problem, opper1995statistical, opper2001learning, skurichina2002bagging, le1991eigenvalues, lecun1991second}
and more recently in \citet{advani2017high, geiger2019scaling, spigler2018jamming, neal2018modern}.
Recently \citet{nakkiran2020deep} demonstrated a generalized double descent phenomenon
on modern deep networks, and highlighted ``sample non-monotonicity'' as an aspect of double descent. 

A recent stream of theoretical
works consider model-wise double descent in simplified settings--- often via linear models
for regression or classification.
This also connects to works on high-dimentional regression in the statistics literature.
A partial list of works in these areas include
\cite{belkin2019two, hastie2019surprises, bartlett2019benign, muthukumar2019harmless, bibas2019new, Mitra2019UnderstandingOP, mei2019generalization, liang2018just, liang2019risk,xu2019number,dereziski2019exact,lampinen2018analytic, deng2019model, nakkiran2019more, mahdaviyeh2019asymptotic,dobriban2018high,dobriban2019wonder,kobak2018optimal}.
Of these, most closely related to our work are
\citet{hastie2019surprises, dobriban2018high, mei2019generalization}.
Specifically, \citet{hastie2019surprises} considers the risk of unregularized
and regularized linear regression in an asymptotic regime, where dimension $d$
and number of samples $n$ scale to infinity together, at a constant ratio $d/n$.
In contrast, we show \emph{non-asymptotic} results,
and are able to consider increasing the number of samples
for a fixed model, without scaling both together.
\citet{mei2019generalization} derive similar results for unregularized and
regularized random features,
also in an asymptotic limit.
The non-asymptotic versions of the settings considered in \citet{hastie2019surprises}
are almost identical to ours--- for example, our projection model in Section~\ref{sec:proj-model} is nearly identical to the model in
\citet[Section 5.1]{hastie2019surprises}.
Finally, subsequent to our work,
\citet{d2020triple} identified triple descent in an asymptotic setting.
\section{Sample Monotonicity in Ridge Ridgression}
\label{sec:sample}
In this section, we prove that optimally-regularized ridge regression 
has test risk that is monotonic in samples,
for isotropic gaussian covariates and linear response.
This confirms the behavior empirically observed in Figure~\ref{fig:isotropic}.
We also show that this monotonicity is not ``fragile'', and using
larger than larger regularization is still sample-monotonic
(consistent with Figure~\ref{fig:isotropic}).

Formally, we consider the following linear regression problem in $d$ dimensions. The input/covariate $x\in \R^d$ is generated from $\cN(0, I_d)$, and the output/response is generated by
$$y = \langle x, \B \rangle + \epsilon$$ with $\epsilon\sim \cN(0, \sigma^2)$ and for some unknown parameter $\B \in \R^d$. We denote the joint distribution of $(x,y)$ by $\cD$. 
We are given $n$ training examples $\{(x_i, y_i)\}_{i=1}^n$ i.i.d sampled from $\cD$. 
We aim to learn a linear model
$f_{\beta}(x) = \langle x, \beta \rangle$ with small population mean-squared error on the distribution $\cD$
\begin{align*}
R(\beta) &:=
\E_{(x, y) \sim \cD}[(\langle x, \beta \rangle - y)^2]
\end{align*}
For simplicity, let $X \in \R^{n \x d}$ be the data matrix that contains $x_i^\top$'s as rows and let $\vec{y} \in \R^n$ be column vector that contains the responses $y_i$'s as entries.
For any estimator $\hbeta_n(X, \vec y)$ as a function of $n$ samples,
define the expected risk of the estimator as:
\begin{align}
\bar{R}(\hbeta_n) := \E_{X, y \sim \cD^n}[R(\hbeta_n(X, \vec y))]
\end{align}

We consider the regularized least-squares estimator,
also known as the ridge regression estimator.
For a given $\lambda > 0$, define
\begin{align}
\hbeta_{n, \lambda} &:= \argmin_\beta ||X\beta - \vec y||_2^2 + \lambda||\beta||_2^2\\
&= (X^TX + \lambda I_d)^{-1}X^T \vec y \label{eqn:beta}
\end{align}

Here $I_d$ denotes the $d$ dimensional identity matrix. Let $\Ln{n}$ be the optimal ridge parameter (that achieves the minimum expected risk) given $n$ samples:
\begin{align}
\Ln{n} &:= \argmin_{\lambda: \lambda \geq 0} \bR(\hbeta_{n, \lambda})) 
\end{align}
Let $\Bn{n}$ be the estimator that corresponds to the $\Ln{n}$
\begin{align}
\Bn{n} &:= \argmin_\beta
||X\beta - \vec y||_2^2 + \Ln{n}||\beta||_2^2
\end{align}

Our main theorem in this section shows that the expected risk of $\Bn{n}$ monotonically decreases as $n$ increases. 
\begin{theorem}
	\label{thm:main1}
	In the setting above, the expected test risk of optimally-regularized well-specified isotropic linear regression is monotonic in samples.
	That is, for all $\beta^* \in \R^d$ and all $d \in \N, n \in \N, \sigma > 0$,
	$$
	\bR(\Bn{n+1})
	\leq 
	\bR(\Bn{n})
	$$
\end{theorem}

The above theorem shows a strong form of monotonicity,
since it holds for every fixed ground-truth $\beta^*$,
and does not require averaging over any prior on ground-truths.
Moreover, it holds \emph{non-asymptotically}, for every fixed $n, d \in \N$.
Obtaining such non-asymptotic results is nontrivial, since we cannot rely
on concentration properties of the involved random variables.

In particular, evaluating $\bR(\Bn{n})$ as a function of the problem parameters ($n,\sigma, \beta^*$, and $d$) is technically challenging. In fact, we suspect that a simple closed form expression does not exist. The key idea towards proving the theorem is to derive a ``partial evaluation'' --- the following lemmas shows that we can write $\bR(\Bn{n})$ in the form of $\E[g(\gamma, \sigma, n, d, \beta^*)]$ where $\gamma\in \R^{d}$ contains the singular values of $X$.
We will then couple the randomness of data matrices
obtained by adding a single sample,
and use singular value interlacing to compare their singular values.

\begin{lemma}
	\label{lem:svd}In the setting of Theorem~\ref{thm:main1}, 
	let $\gamma = (\gamma_1, \dots, \gamma_d)$ be the singular values of the data matrix $X\in \R^{n \x d}$. (If $n < d$, we pad the $\gamma_i =0$ for $i > n$.)
	Let $\Gamma_n$ be the distribution of $\gamma$. 
	Then, the expected test risk is
	\begin{align*}
	\bR(\hbeta_{n, \lambda})
	= \E_{(\gamma_1, \dots \gamma_d) \sim \Gamma_n}\left[ \sum_{i=1}^d \frac{||\B||_2^2 \lambda^2 / d + \sigma^2\gamma_i^2}{(\gamma_i^2+\lambda)^2}  \right]
	+\sigma^2
	\end{align*}
	
\end{lemma}

From Lemma~\ref{lem:svd}, the below lemma follows directly
by taking derivatives to find the optimal $\lambda$.

\begin{lemma}\label{lem:optrisk}
	In the setting of Theorem~\ref{thm:main1}, the optimal ridge parameter is constant for all $n$:
	$\Ln{n} = \frac{d\sigma^2}{||\B||_2^2}.$ Moreover, the optimal expected test risk can be written as
		\begin{align}
	\bR(\Bn{n})
	= \E_{(\gamma_1, \dots \gamma_d) \sim \Gamma_n}\left[ \sum_{i=1}^d \frac{\sigma^2}{\gamma_i^2+d\sigma^2/||\B||_2^2}  \right]
	+\sigma^2
	\label{eqn:optrisk}
	\end{align}
\end{lemma}

Lemma~\ref{lem:optrisk}'s proof is deferred to the Appendix, Section~\ref{sec:sampleproofs}.
We now prove Lemma~\ref{lem:svd}.
\begin{proof}[Proof of Lemma~\ref{lem:svd}]
	For isotropic $x$, the test risk is related to the parameter error as:
	\begin{align*}
	R(\hbeta) &:=
	\E_{(x, y) \sim \cD}[(\langle x, \hbeta \rangle - y)^2]\\
	&= \E_{x \sim \cN(0, I_d), \eta \sim \cN(0, \sigma^2)}
	[(\langle x, \hbeta - \B\rangle  + \eta)^2]\\
	&= ||\hbeta - \B||_2^2 + \sigma^2
	\end{align*}
	Plugging in the form of $\hbeta_{n, \lambda}$ and expanding:
	\begin{align*}
	&\bR(\hbeta_{n, \lambda}) := \E_{X, y \sim \cD^n}[R(\hbeta_{n, \lambda})]\\
	&= \E[ ||\hbeta_{n, \lambda} - \B||_2^2 ] + \sigma^2\\
	&= \E_{X, y}[ ||(X^TX + \lambda I)^{-1}X^T y - \beta||_2^2 ] + \sigma^2\\
	&= \E_{X, \eta \sim \cN(0, \sigma^2 I_n)}[ ||(X^TX + \lambda I)^{-1}X^T (X\B + \eta) - \B||_2^2 ] + \sigma^2\\
	&= \E_{X}[ ||(X^TX + \lambda I)^{-1}X^T X\B - \B||_2^2 ]
	+ \E_{X, \eta}[ ||(X^TX + \lambda I)^{-1}X^T \eta ||_2^2] + \sigma^2\\
	&= \E_{X}[ ||(X^TX + \lambda I)^{-1}X^T X\B - \B||_2^2 ]
	+ \sigma^2 \E_{X}[ ||(X^TX + \lambda I)^{-1}X^T||_F^2] + \sigma^2
	\end{align*}
	Now let $X = U \Sigma V^T$ be the full singular value decomposition of $X$,
	with $U \in \R^{n \x n}, \Sigma \in \R^{n \x d},  V \in \R^{d \x d}$.
	Let $(\gamma_1, \dots \gamma_d)$
	denote the singular values, defining $\gamma_i = 0$ for $i > \min(n, d)$.
	Then, continuing:
	\begin{align}
	\bR(\hbeta_{n, \lambda}) &=
	\E_{V, \Sigma}[||\diag(\{\frac{-\lambda}{\gamma_i^2+\lambda}\})V^T \B ||_2^2]
	+ \sigma^2\E_{\Sigma}[ \sum_i \frac{\gamma_i^2}{(\gamma_i^2+\lambda)^2} ]
	+\sigma^2\\
	&= \E_{z \sim \mathrm{Unif}(||\B||_2\mathbb{S}^{d-1}), \Sigma}[||\diag(\{\frac{-\lambda}{\gamma_i^2+\lambda}\})z ||_2^2]
	+ \sigma^2\E_{\Sigma}[ \sum_i \frac{\gamma_i^2}{(\gamma_i^2+\lambda)^2} ]
	+\sigma^2
	\label{eqn:sphere} \\
	&= \frac{||\B||_2^2}{d} \E_{\Sigma}[\sum_i \frac{\lambda^2}{(\gamma_i^2+\lambda)^2}]
	+ \sigma^2\E_{\Sigma}[ \sum_i \frac{\gamma_i^2}{(\gamma_i^2+\lambda)^2} ] 
	+\sigma^2\\
	&= \E_{\Sigma}[ \sum_i \frac{||\B||_2^2 \lambda^2 / d + \sigma^2 \gamma_i^2}{(\gamma_i^2+\lambda)^2} ]
	+\sigma^2
	\end{align}
	In Line~\eqref{eqn:sphere} follows because by symmetry, the distribution of $V$
	is a uniformly random orthonormal matrix, and $\Sigma$ is independent of $V$.
	Thus, $z := V^T \B$ is distributed as a uniformly random point on the unit sphere of radius $||\B||_2$.
	
\end{proof}

Now we are ready to prove Theorem~\ref{thm:main1}.

\begin{proof}[Proof of Theorem~\ref{thm:main1}]

	Let $\wt{X} \in \R^{(n+1) \x d}$ and $X \in \R^{n \x d}$
	be any two matrices which differ by only the last row of $\wt{X}$. 
	By the Cauchy interlacing theorem Theorem 4.3.4 of ~\citet{horn1990matrix} (c.f.,Lemma 3.4 of ~\citet{marcus2014ramanujan}), the singular values of $X$ and $\wt{X}$
	are interlaced: $\forall i: \gamma_{i-1}(X)\ge \gamma_i(\wt{X}) \geq \gamma_i(X)$
	where $\gamma_i(\cdot)$ is the $i$-th singular value. 
	
	If we couple $\wt{X}$ and $X$, it will induce a coupling $\Pi$ between the distributions
	$\Gamma_{n+1}$ and $\Gamma_{n}$, of the singular values of the data matrix for $n+1$ and $n$ samples. This coupling satisfies that 
	$\wt\gamma_i \geq \gamma_i$ with probability 1 for $(\{\wt\gamma_i\}, \{\gamma_i\}) \sim \Pi$.
	
	Now, expand the test risk using Lemma~\ref{lem:optrisk},
	and observe that each term in the sum of Equation~\eqref{eqn:lineRn} below
	is monotone decreasing with $\gamma_i$. Thus:
	\begin{align}
	\bR(\Bn{n})
	&= \E_{(\gamma_1, \dots \gamma_d) \sim \Gamma_n}\left[ \sum_{i=1}^d \frac{\sigma^2}{\gamma_i^2+d\sigma^2/||\B||_2^2}  \right]
	+\sigma^2 \label{eqn:lineRn} \\
	&\geq \E_{(\wt\gamma_1, \dots \wt\gamma_d) \sim \Gamma_{n+1}}\left[ \sum_{i=1}^d \frac{\sigma^2}{\wt\gamma_i^2+d\sigma^2/||\B||_2^2}  \right]
	+\sigma^2\\
	&= \bR(\Bn{n+1})
	\end{align}
\end{proof}

By similar techniques, we can also prove that \emph{overregularization} ---that is,
using ridge parameters $\lambda$ larger than the optimal value--- is still monotonic.
This proves the behavior empirically observed in Figure~\ref{fig:isotropic}.

\begin{theorem}
\label{thm:overreg}
	In the same setting as Theorem~\ref{thm:main1},
	over-regularized regression is also monotonic in samples.
	That is, for all $d \in \N, n \in \N, \sigma > 0, \B \in \R^d$,
	the following holds
	$$
	\forall \lambda \geq \lambda^* :\quad
	\bR(\hbeta_{n+1, \lambda})
	\leq 
	\bR(\hbeta_{n, \lambda})
	$$
	where $\lambda^* = \frac{d\sigma^2}{||\B||_2^2}$.
\end{theorem}
\begin{proof}
In Section~\ref{sec:sampleproofs}.
\end{proof}

\section{Model-wise Monotonicity in Ridge Regression}
\label{sec:proj-model}
In this section, we show that for a certain family of linear models,
optimal regularization prevents model-wise double descent.
That is, for a fixed number of samples, larger models
are not worse than smaller models.

We consider the following learning problem.
Informally, covariates live in a $p$-dimensional ambient space,
and we consider models which first linearly project down to 
a random $d$-dimensional subspace, then perform ridge regression in that subspace for some $d\le p$.

Formally, the covariate $x\in \R^p$ is generated from $\cN(0, I_p)$, and the response is generated by
$$y = \langle x, \theta \rangle + \epsilon$$ with $\epsilon\sim \cN(0, \sigma^2)$ and for some unknown parameter $\theta \in \R^p$.
Next, $n$ examples $\{(x_i, y_i)\}_{i=1}^n$ are sampled i.i.d from this distribution.
For a given model size $d \leq p$,
we first sample a random orthonormal matrix
$P \in \R^{d \x p}$ which specifies our model.
We then consider models
which operate on $(\t{x_i}, y_i) \in \R^d \times \R$, 
where $\t{x_i} = Px_i$.
We denote the joint distribution of $(\t{x},y)$ by $\cD$. Here, we emphasize that $p$ is some large ambient dimension and $d \leq p$ is the size of the model we learn.

For a fixed $P$, we want to learn a linear model
$f_{\hbeta}(\tilde{x}) = \langle \tilde{x}, \hbeta \rangle$
for estimating $y$,
with small mean squared error on distribution:
\begin{align*}
R_P(\hbeta) &:=
\E_{(\t{x}, y) \sim \cD}[(\langle \tilde{x}, \hbeta \rangle - y)^2]
= \E_{(x, y)}[(\langle Px, \hbeta \rangle - y)^2]
\end{align*}
For $n$ samples $(x_i, y_i)$, let $X \in \R^{n \x p}$
be the data matrix,
$\t{X} = XP^T \in \R^{n \times d}$ be the projected data matrix and $\vec{y} \in \R^n$ be the responses.
For any estimator $\hbeta(\t{X}, \vec y)$
as a function of the observed samples,
define the expected risk of the estimator as:
\begin{align}
\bar{R}(\hbeta) := \E_P \E_{\t{X}, \vec y \sim \cD^n}[R_P(\hbeta(\Tilde{X}, \vec y)]
\end{align}

We consider the regularized least-squares estimator.
For a given $\lambda > 0$, define
\begin{align}
\hbeta_{d, \lambda} &:= \argmin_\beta ||\t{X}\beta - \vec y||_2^2 + \lambda||\beta||_2^2\\
&= (\t{X}^T\t{X} + \lambda I_d)^{-1}\t{X}^T \vec y
\end{align}

Let $\Ln{d}$ be the optimal ridge parameter (that achieves the minimum expected risk) for a model of size $d$,
with $n$ samples:
\begin{align}
\Ln{d} &:= \argmin_{\lambda \geq 0}\bR(\hbeta_{d, \lambda})) 
\end{align}
Let $\Bn{d}$ be the estimator that corresponds to the $\Ln{d}$
\begin{align}
\Bn{d} &:= \argmin_\beta
||\t{X}\beta - \vec y||_2^2 + \Ln{d}||\beta||_2^2
\end{align}

Now, our main theorem in this setting shows that with optimal $\ell_2$ regularization,
test performance is monotonic in model size.
\begin{theorem}
\label{thm:main-proj}
In the setting above, the expected test risk of
the optimally-regularized model
is monotonic in the model size $d$.

That is, for all $p \in \N, \theta \in \R^p, d \leq p, n \in \N, \sigma > 0$,
we have
$$
\bR(\Bn{d+1})
\leq 
\bR(\Bn{d})
$$
\end{theorem}
\begin{proof}
In Section~\ref{sec:proj-proofs}.
\end{proof}

This proof follows closely the proof of Theorem~\ref{thm:main1},
making crucial use of Lemma~\ref{lem:opt-proj} below.

\begin{lemma}
\label{lem:opt-proj}
For all $\theta \in \R^p$, $d,n \in \N$, and $\lambda > 0$,
let $X \in \R^{n \x p}$ be a matrix with i.i.d. $\cN(0, 1)$
entries.
Let $P \in \R^{d \x p}$ be a random orthonormal matrix.
Define $\t{X} := XP^T$. %

Let $(\gamma_1, \dots, \gamma_m)$ be the singular values of the data matrix $\Tilde{X}\in \R^{n \x d}$,
for $m := \max(n, d)$
(with $\gamma_i = 0$ for $i > \min(n, d)$).
Let $\Gamma_d$ be the distribution of singular values
$(\gamma_1, \dots, \gamma_m)$. 

Then,
the optimal ridge parameter
is constant for all $d$:
\begin{align*}
\Ln{d}
= \frac{p^2 \t{\sigma}^2}{d ||\theta||_2^2}
\end{align*}
where we define
\begin{align*}
\t{\sigma}^2 := 
\sigma^2 + \frac{p-d}{p}||\theta||_2^2
\end{align*}

Moreover, the optimal expected test risk can be written as
\begin{align*}
\bar{R}(\Bn{d})
&=
\t{\sigma}^2
+
\E_{(\gamma_1, \dots, \gamma_m) \sim \Gamma_d}
\left[\sum_{i=1}^p
\frac{\t{\sigma}^2}
{
\gamma_i^2
+
\frac{\t{\sigma}^2p^2}{d||\theta||_2^2}
}
\right]
\end{align*}
\end{lemma}
\begin{proof}
This proof follows exactly analogously
as the proof of Lemma~\ref{lem:optrisk}
from Lemma~\ref{lem:svd}, in Section~\ref{sec:sampleproofs}.
\end{proof}

\section{Counterexamples to Monotonicity}
In this section, we show that optimally-regularized ridge regression
is \emph{not} always monotonic in samples.
We give a numeric counterexample
in $d=2$ dimensions, with non-gaussian covariates and heteroscedastic noise.
This does not contradict our main theorem in Section~\ref{sec:sample},
since this distribution is not jointly Gaussian with isotropic marginals.

\subsection{Counterexample}\label{sec:c1}
Here we give an example of a distribution $(x, y)$ for which
the expected error of optimally-regularized ridge regression
with $n=2$ samples is worse than with $n=1$ samples.

This counterexample is most intuitive to understand
when the ridge parameter $\lambda$
is allowed to depend on the specific sample instance $(X, \vec{y})$ as well as $n$\footnote{
Recall, our model of optimal ridge regularization from Section~\ref{sec:sample}
only allows $\lambda$ to depend on $n$ (not on $X, \vec y$).}.
We sketch the intuition for this below.

Consider the following distribution on $(x, y)$ in $d=2$ dimensions.
This distribution has one ``clean'' coordinate and one ``noisy'' coordinate.
The distribution is:
\begin{align*}
    (x, y) \sim 
    \begin{cases}
    (\vec e_1, 1) & \text{w.p.} 1/2\\
    (\vec e_2, \pm A) & \text{w.p.} 1/2\\
    \end{cases}
\end{align*}
where $A=10$
and $\pm A$ is uniformly random independent noise.
This distribution is ``well-specified'' in that the optimal
predictor is linear in $x$:
$\E[ y | x] = \langle \beta^*, x \rangle$
for $\beta^* = [1 , 0]$.
However, the noise is heteroscedastic.

For $n=1$ samples, the estimator can decide
whether to use small $\lambda$ or large $\lambda$
depending on if the sampled coordinate is the ``clean'' or ``noisy'' one.
Specifically, for the sample $(x, y)$:
If $x = \vec e_1$, then the optimal ridge parameter
is $\lambda=0$.
If $x = \vec e_2$, then the optimal parameter is 
$\lambda = \infty$.

For $n=2$ samples, with probability $1/2$
the two samples will hit both coordinates.
In this case, the estimator must chose a
single value of $\lambda$ uniformly for both coordinates.
This yields to a suboptimal tradeoff, since
the ``noisy'' coordinate demands large regularization,
but this hurts estimation on the ``clean'' coordinate.

It turns out that a slight modification to
the above also serves as a counterexample
to monotonicity when the regularization parameter $\lambda$
is chosen only depending on $n$ (and not on the instance $X, y$).

The distribution is:
\begin{align*}
    (x, y) \sim 
    \begin{cases}
    (\vec e_1, 1) & \text{w.p. } 1-\eps\\
    (\vec e_2, \pm A) & \text{w.p. } \eps\\
    \end{cases}
\end{align*}
with $A = 20$, $\eps = 0.02$.

\begin{theorem}
There exists a distribution $\cD$
over $(x, y)$ for $x \in \R^2, y \in \R$
with the following properties.

Let $\Bn{n}$ be the optimally-regularized ridge regression solution for $n$ samples $(X, \vec y)$ from $\cD$.
Then:
\begin{enumerate}
    \item $\cD$ is ``well-specified''
    in that 
$\E_\cD[ y | x]$ is a linear function of $x$,
\item The expected test risk increases as a function of $n$, between $n=1$
and $n=2$.
Specifically
$$
\bR(\Bn{n=1})
< 
\bR(\Bn{n=2})
$$
\end{enumerate}
\end{theorem}
\begin{proof}
For $n=1$ samples,
it can be confirmed analytically that
the expected risk
$\bR(\Bn{n=1})
< 8.157$.
This is achieved with $\lambda = 400/2401 \approx 0.166597$.

For $n=2$ samples,
it can be confirmed numerically (via Mathematica)
that the expected risk
$\bR(\Bn{n=2})
> 8.179$.
This is achieved with $\lambda = 0.642525$.
\end{proof}

\section{Experiments}\label{sec:experiment}
We now experimentally demonstrate that
optimal $\ell_2$ regularization can mitigate double descent,
in more general settings than Theorems~\ref{thm:main1} and \ref{thm:main-proj}.

\subsection{Sample Monotonicity}
\label{sec:exp-samples}

Here we show various settings where
optimal $\ell_2$ regularization
empirically induces sample-monotonic performance.

\paragraph{Nonisotropic Regression.}
We first consider the setting of Theorem~\ref{thm:main1},
but with non-isotropic covariantes $x$.
That is, we perform ridge regression on samples $(x, y)$,
where the covariate $x\in \R^d$ is generated from $\cN(0, \Sigma)$
for $\Sigma \neq I_d$.
As before, the response is generated by $y = \langle x, \B \rangle + \epsilon$ with $\epsilon\sim \cN(0, \sigma^2)$ for some unknown parameter $\B \in \R^d$.

We consider the same ridge regression estimator,
\begin{align}
\hbeta_{n, \lambda} &:= \argmin_\beta ||X\beta - \vec y||_2^2 + \lambda||\beta||_2^2
\end{align}

\begin{figure}[h]
    \centering
    \includegraphics[width=0.6\textwidth]{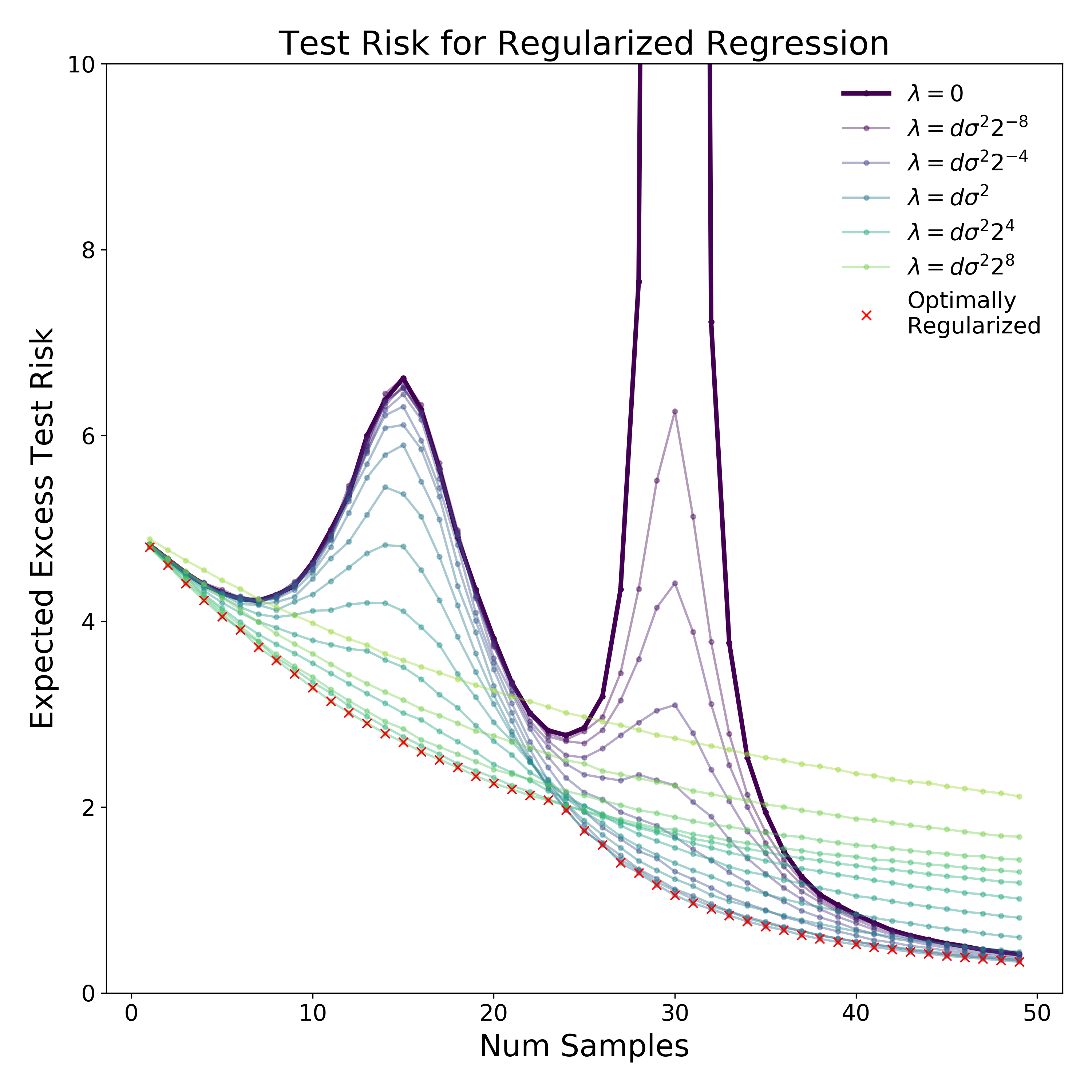}
    \caption{{\bf Test Risk vs. Num. Samples for Non-Isotropic Ridge Regression in $d=30$ dimensions.}
    Unregularized regression is non-monotonic in samples,
    but optimally-regularized regression is monotonic.
    Note the optimal regularization $\lambda$ depends on the number of samples $n$.
    Plotting empirical means of test risk over $5000$ trials. See Figure~\ref{fig:nonisotropic-train}
    for the corresponding train errors.
    }
    \label{fig:nonisotropic}
\end{figure}

Figure~\ref{fig:nonisotropic} shows one instance of this, for a particular choice of $\Sigma$
and $\beta^*$.
The covariance $\Sigma$ is diagonal, with $\Sigma_{i, i} = 10$ for $i \leq 15$
and $\Sigma_{i, i} = 1$ for $i > 15$.
That is, the covariance has one ``large'' eigenspace
and one ``small'' eigenspace.
The ground-truth $\beta^* = 0.1 \vec{e_1} + \vec{e_{30}}$, which
lies almost entirely within the ``small'' eigenspace of $\Sigma$.
The noise parameter is $\sigma=0.5$.

We see that unregularized regression ($\lambda=0$) actually undergoes ``triple descent''\footnote{See also
the ``multiple descent'' behavior of kernel interpolants in ~\citet{liang2020multiple}.}
in this setting, with the first peak around $n=15$ samples
due to the 15-dimensional large eigenspace, and the second peak at $n=d$.

In this setting, optimally-regularized
ridge regression is empirically monotonic in samples (Figure~\ref{fig:nonisotropic}).
Unlike the isotropic setting of Section~\ref{sec:sample},
the optimal ridge parameter $\lambda_n$ is no longer a constant, but varies with
number of samples $n$.

\paragraph{Random ReLU Features.}
We consider random ReLU features, in the random features framework
of~\cite{rahimi2008random}.
We apply random features to
Fashion-MNIST~\cite{xiao2017fashion},
an image classification problem with 10 classes.
Input images $x \in \R^{d}$ are normalized and flattened
to $[-1, 1]^{d}$ for $d = 784$.
Class labels are encoded as one-hot vectors $y \in \{\vec{e_1}, \dots \vec{e_{10}}\} \subset \R^{10}$.
For a given number of features $D$, and number of samples $n$,
the random feature classifier is obtained by
performing regularized linear regression on the embedding
$$\tilde{x} := \mathrm{ReLU}(W x)$$
where $W \in \R^{D \x d}$ is a matrix with each
entry sampled i.i.d $\cN(0, 1/\sqrt{d})$,
and $\textrm{ReLU}$ applies pointwise.
This is equivalent to a 2-layer fully-connected neural network
with a frozen (randomly-initialized) first layer,
trained with $\ell_2$ loss and weight decay.

Figure~\ref{fig:relu-error} shows
the test error of the random features classifier,
for $D=500$ random features and varying number of train samples.
We see that underregularized models are non-monotonic, but
optimal $\ell_2$ regularization is monotonic in samples.
Moreover, the optimal ridge parameter $\lambda$ appears to be constant
for all $n$, similar to our results from the
isotropic setting in Theorem~\ref{thm:main1}.

\begin{figure}[h]
    \centering
    \begin{subfigure}[t]{0.45\textwidth}
         \centering
     \includegraphics[width=\textwidth]{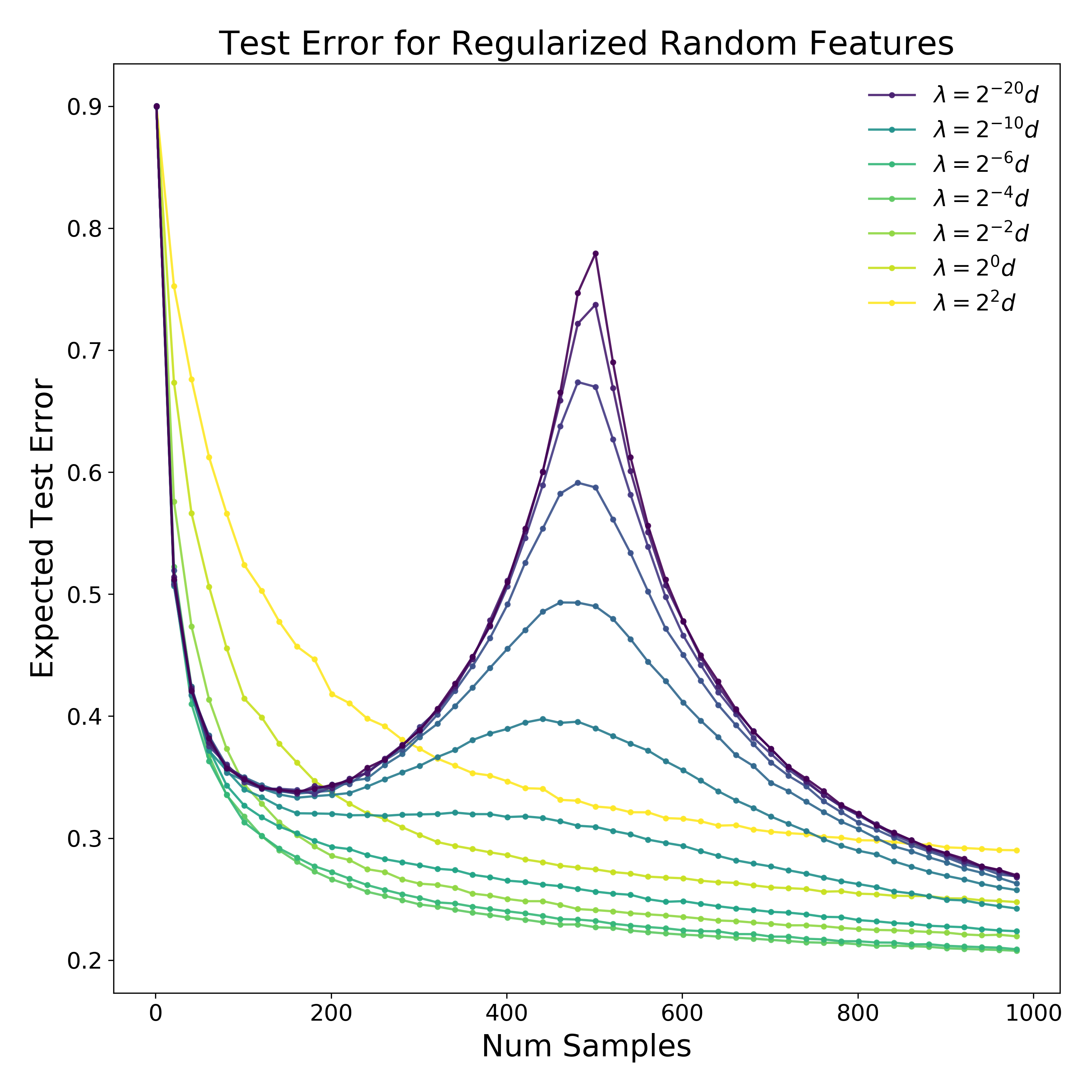}
    \caption{
     Test Classification Error vs. Number of Training Samples.
    }
         \label{fig:relu-error}
     \end{subfigure}
     \hfill
    \begin{subfigure}[t]{0.45\textwidth}
         \centering
    \includegraphics[width=\textwidth]{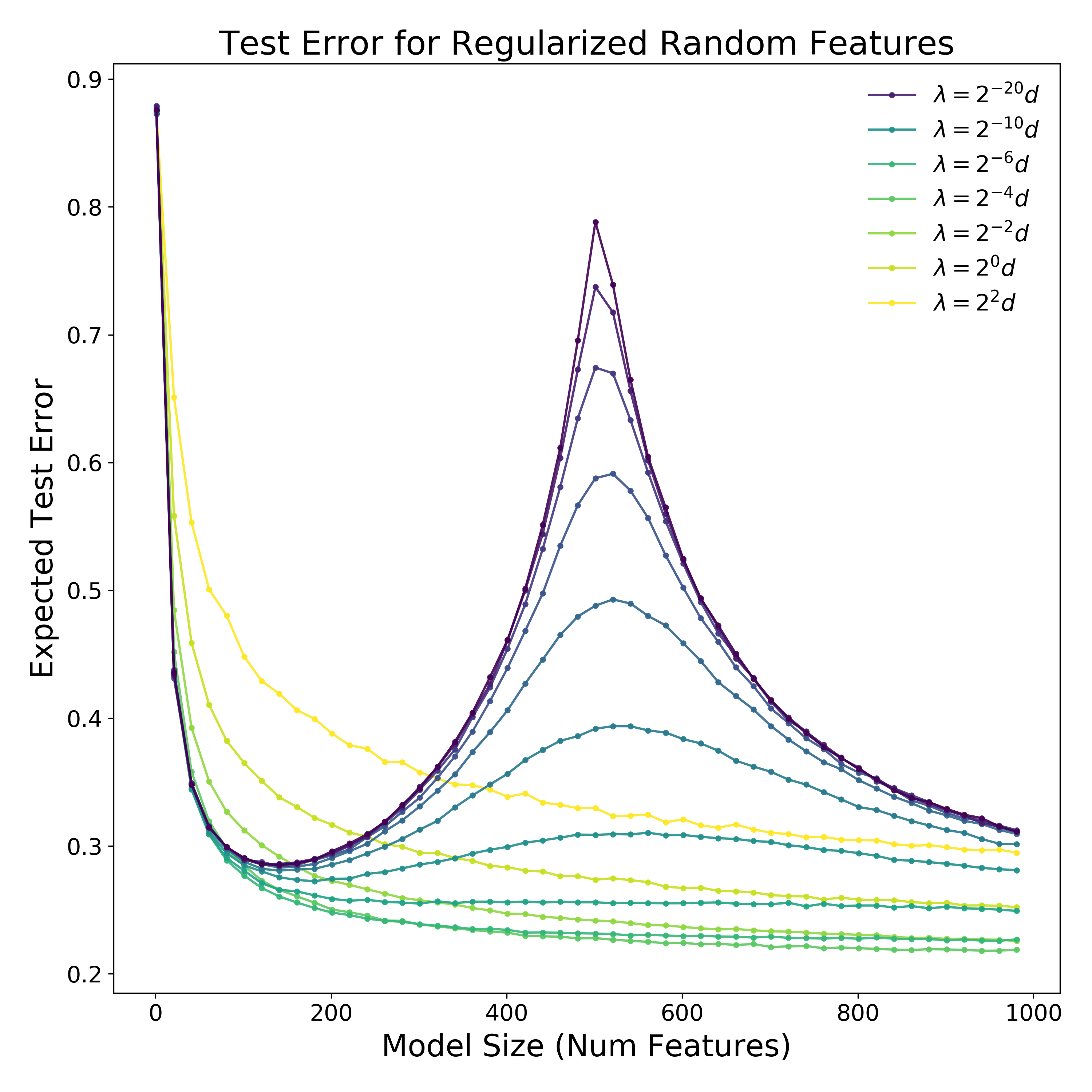}
    \caption{
    Test Classification Error vs. Model Size
    (Number of Random Features).
    }
    \label{fig:relu-model-error}
     \end{subfigure}
     
     \caption{
     {\bf Double-descent for Random ReLU Features.}
     Test classification error 
     as a function of model size and sample size
    for Random ReLU Features on Fashion-MNIST.
    Left: with $D=500$ features. Right: with $n=500$ samples.
    See Figures~\ref{fig:relu-samps-mse}, \ref{fig:relu-model-mse}
    for the  corresponding test Mean Squared Error.
    See Appendix D of \citet{nakkiran2020deep}
    for the performance of these unregularized
    models
    plotted across Num. Samples $\x$ Model Size
    simultaneously.
     }
\end{figure}

\subsection{Model-size Monotonicity}
Here we empirically show that optimal $\ell_2$ regularization can mitigate model-wise double descent.

\paragraph{Random ReLU Features.}
We consider the same experimental setup 
as in Section~\ref{sec:exp-samples},
but now fix the number of samples $n$,
and vary the number of random features $D$.
This corresponds to varying the width of the corresponding
2-layer neural network.

Figure~\ref{fig:relu-model-error} shows
the test error of the random features classifier,
for $n=500$ train samples and varying number of random features.
We see that underregularized models undergo model-wise double descent,
but optimal $\ell_2$ regularization prevents double descent.

\paragraph{Convolutional Neural Networks.}
We follow the experimental setup of \citet{nakkiran2020deep}
for model-wise double descent, and add varying
amounts of $\ell_2$ regularization (weight decay).
We chose the following setting from \citet{nakkiran2020deep},
because it exhibits double descent even with no added label noise.

We consider the same family of 5-layer convolutional neural networks (CNNs)
from \citet{nakkiran2020deep}, consisting of 4 convolutional layers
of widths $[k, 2k, 4k, 8k]$ for varying $k \in \N$.
This family of CNNs was introduced by~\citet{mcnn}.
We train and test on CIFAR-100~\cite{krizhevsky2009learning},
an image classification problem with 100 classes.
Inputs are normalized to $[-1, 1]^d$,
and we use standard data-augmentation of random horizontal flip
and random crop with 4-pixel padding.
All models are trained using Stochastic Gradient Descent (SGD) on the cross-entropy loss,
with step size $0.1/\sqrt{\lfloor T / 512 \rfloor + 1}$ at step $T$.
We train for $1e6$ gradient steps, and use weight decay $\lambda$ for varying $\lambda$.
Due to optimization instabilities for large $\lambda$,
we use the model with the minimum train loss
among the last 5K gradient steps.

\begin{figure}[h]
    \centering
    \includegraphics[width=0.7\textwidth]{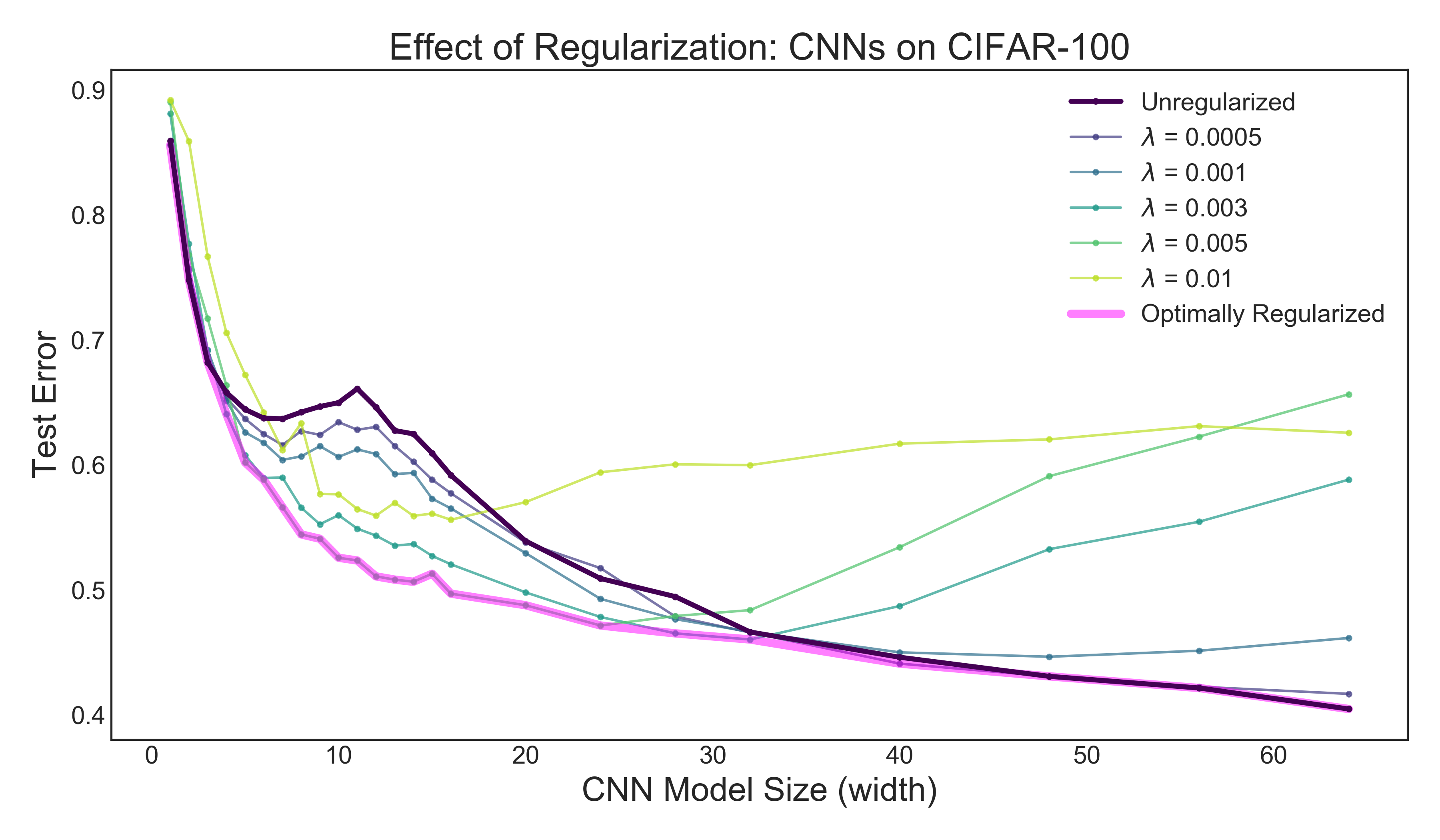}

    \caption{{\bf Test Error vs. Model Size for 5-layer CNNs on CIFAR-100},
    with $\ell_2$ regularization (weight decay).
    Note that the optimal regularization $\lambda$ varies with $n$.
    See Figure~\ref{fig:cifar-train} for the corresponding train errors.
    }
    \label{fig:cifar}
\end{figure}

Figure~\ref{fig:cifar} shows the test error of these models on CIFAR-100.
Although unregularized and under-reguarized models exhibit double descent,
the test error of optimally-regularized models is largely monotonic.
Note that the optimal regularization $\lambda$ varies with the model size ---
no single regularization value is optimal for all models.

\section{Towards Monotonicity with General Covariates}
\label{sec:general}
Here we investigate whether monotonicity
provably holds in more general models,
inspired by the experimental results.
As a first step, we consider Gaussian (but not isotropic) covariances
and homeostatic noise.
That is, we consider ridge regression in the setting of
Section~\ref{sec:sample},
but with $x \sim \cN(0, \Sigma)$,
and $y \sim \langle x, \beta^* \rangle + N(0, \sigma^2)$.
In this section, we observe that ridge regression can be made sample-monotonic with a modified regularizer.
We also conjecture that ridge regression
is sample-monotonic without modifying the regularizer, and we outline a potential proof strategy along with numerical evidence.

\subsection{Adaptive Regularization}
\label{sec:unsupervised_reg}
The results on isotropic regression
in Section~\ref{sec:sample}
imply that ridge regression can be made sample-monotonic
even for non-isotropic covariates, if an appropriate regularizer is applied.
Specifically, the appropriate regularizer depends on the covariance of the inputs:
for $x \sim \cN(0, \Sigma)$, the following estimator is sample-monotonic
for optimally-tuned $\lambda$:
\begin{align}
    \hbeta_{n, \lambda} := \argmin_\beta
    ||X \beta - \vec{y} ||_2^2 + \lambda ||\beta||^2_{\Sigma^{-1}}
\end{align}
This follows directly from Theorem~\ref{thm:main1} by applying a change-of-variable;
full details of this equivalence are in Section~\ref{sec:reduction}.
Note that if the population covariance $\Sigma$
is not known, it can potentially be estimated from unlabeled data.

\subsection{Towards Proving Monotonicity}
We conjecture that optimally-regularized
ridge regression is sample-monotonic
for non-isotropic covariates,
even without modifying the regularizer
(as suggested by the experiment in Figure~\ref{fig:nonisotropic}).
We derive a sufficient condition for
monotonicity, which we have numerically verified in a variety of instances.

Specifically, we conjecture the following.

\begin{conjecture}
\label{conj:main}
For all $d \in \N$, and all PSD covariances
$\Sigma \in \R^{d \x d}$, consider the distribution on $(x, y)$
where $x \sim \cN(0, \Sigma)$,
and $y \sim \langle x, \beta^* \rangle + \cN(0, \sigma^2)$.
Then, we conjecture that the expected test risk
of the ridge regression estimator:
\begin{align}
    \hbeta_{n, \lambda} := \argmin_\beta
    ||X \beta - \vec{y} ||_2^2 + \lambda ||\beta||^2_{2}
\end{align}
for optimally-tuned $\lambda \geq 0$, is monotone non-increasing in number of samples $n$.
That is, for all $n \in \N$,
\begin{align}
\label{eqn:infs}
\inf_{\lambda \ge 0}
\bR(\hbeta_{n+1, \lambda})
~\leq~
\inf_{\lambda \ge 0}
\bR(\hbeta_{n, \lambda})
\end{align}
where we define $\hbeta_{n, 0}
:= \lim_{\lambda \rightarrow 0+} \hbeta_{n, \lambda} = X^{\dagger}y$.
\end{conjecture}

In order to establish Conjecture~\ref{conj:main},
it is sufficient to prove the following
technical conjecture.
\begin{conjecture}
\label{conj:numeric}
For all $n \in \N$,
$d \geq n$,
$\lambda > 0$,
symmetric positive definite matrix $Q \in \R^{d \x d}$,
the following holds.

Define
$$
    G_\lambda^{n} := \lambda^2 \E_{X}[ (X^TX + \lambda Q)^{-2}]
$$
where $X \in \R^{n \x d}$ is sampled with each entry i.i.d. $\cN(0, 1)$.
Similarly, define
$$
H_\lambda^{n} :=
\E_{X}
[||X^TX + \lambda Q)^{-1}X^T||_F^2]
$$
The expected test risk for $n$ samples can be expressed as:
\begin{align}
\bR(\hbeta_{n, \lambda})
=
(\B)^T G_\lambda^n \B
+
\sigma^2 H_\lambda^n
+ \sigma^2
\end{align}

Then, we conjecture that the following two conditions hold.
\begin{enumerate}
\item
\begin{equation}
G_\lambda^n  \succeq G_\lambda^{n+1}
\end{equation}

    \item 
\begin{align}
\label{eqn:conj}
(G_\lambda^{n} - G_\lambda^{n+1})
-
(H_\lambda^n - H_\lambda^{n+1})
\frac{
d G_\lambda^n/d\lambda
}
{
d H_\lambda^n/d\lambda
}
\succeq 0
\end{align}
\end{enumerate}

\end{conjecture}
Proving Conjecture~\ref{conj:numeric}
presents a number of technical challenges,
but we have numerically
verified it in a variety of cases. (One can numerically verify the conjecture for a fixed $Q$, $n$ and $d$. Here $Q$ can be assumed to be diagonal w.l.o.g. because $X$ is isotropic. The matrices and scalars in equation~\eqref{eqn:conj} can be evaluated by sampling the random matrix $X$. The derivatives w.r.t $\lambda$ can be done by auto-differentiation).

It can also be shown that Conjecture~\ref{conj:numeric}
is true when $Q=I$, corresponding to isotropic covariates. %
We show that Conjecture~\ref{conj:numeric}
implies Conjecture~\ref{conj:main} in
in Section~\ref{sec:monoconj} of the Appendix.

\section{Discussion and Conclusion}
In this work, we study the double descent phenomenon
in the context of optimal regularization.
We show that, while unregularized or under-regularized
models often have non-monotonic behavior,
appropriate regularization can eliminate this effect.

Theoretically, we prove that for certain linear regression models with isotropic covariates, optimally-tuned $\ell_2$ regularization achieves monotonic test performance as we grow either the sample size or the model size.
These are the first non-asymptotic monotonicity results we are aware of in linear regression.
We also demonstrate empirically that optimally-tuned $\ell_2$ regularization can mitigate double descent for more general models, including neural networks.
We hope that our results can motivate
future work on mitigating the double descent phenomenon,
and allow us to train high performance models
which do not exhibit unexpected nonmonotonic behaviors.

\paragraph{Open Questions.} 
Our work suggests a number of natural open questions.
First, it is open to prove (or disprove) that optimal ridge regression
is sample-monotonic for non-isotropic Gaussian covariates (Conjecture~\ref{conj:main}). We conjecture that it is, and outline a potential route to proving this (via Conjecture~\ref{conj:numeric}).
The non-isotropic setting presents a number of differences from the isotropic one (e.g. the optimal regularizer $\lambda$ depends on number of samples $n$),
and thus a proof of this may yield further insight into mechanisms of monotonicity.

Second, more broadly, it is open to prove sample-wise or model-wise monotonicity for more general (non-linear) models with appropriate regularizers.
Addressing the monotonicity of non-linear models may require us to design new regularizers which improve the generalization when the model size is close to the sample size. It is possible that data-dependent regularizers
(which depend on certain statistics of the labeled or unlabeled data)
can be used to induce sample monotonicity, analogous to the approach in Section~\ref{sec:unsupervised_reg} for linear models. Recent work has introduced data-dependent regularizers for deep models with improved generalization upper bounds~\cite{wei2019data, wei2019improved}, however a precise characterization of the test risk remain elusive. 

Finally, it is open to understand why large neural networks in practice
are often sample-monotonic in realistic regimes of sample sizes,
even without careful choice of regularization.

\subsection*{Acknowledgements}
Work supported in part by the Simons Investigator Awards of Boaz Barak
and Madhu Sudan, and NSF Awards under grants CCF 1715187, CCF 1565264 and CNS 1618026. Sham Kakade acknowledges funding from the Washington Research Foundation for Innovation in Data-intensive Discovery, and the NSF Awards CCF-1703574, and CCF-1740551.

The numerical experiments were supported in part by Google Cloud research credits,
and a gift form Oracle. The work is also partially supported by SDSI and SAIL at Stanford.

\bibliography{refs}

\bibliographystyle{icml2020}
\newpage

\appendix

\section{Appendix}
In Section~\ref{sec:sampleproofs} and ~\ref{sec:proj-proofs}
we provide the proofs for sample-monotonicity
and model-size monotonicity.
In Section~\ref{sec:plots} we include additional and omitted plots.

\subsection{Sample Monotonicity Proofs}
\label{sec:sampleproofs}
Next we prove Lemma~\ref{lem:optrisk}.

\begin{proof}[Proof of Lemma~\ref{lem:optrisk}]
	First, we determine the optimal ridge parameter.
	Using Lemma~\ref{lem:svd}, we have
	\begin{align*}
	&\frac{\partial}{\partial\lambda}
	\bR(\hbeta_{n, \lambda})
	=
	\frac{\partial}{\partial\lambda}
	\E_{(\gamma_1, \dots \gamma_d) \sim \Gamma}\left[ \sum_i
	\frac{||\beta||_2^2 \lambda^2 / d + \sigma^2\gamma_i^2}{(\gamma_i^2+\lambda)^2}
	\right]\\
	&=
	2(||\B||_2^2 \lambda / d - \sigma^2)
	\underbrace{
		\E_{(\gamma_1, \dots \gamma_d) \sim \Gamma}\left[
		\sum_i \frac{\gamma_i^2}{(\gamma_i^2 + \lambda)^3}
		\right]}_{> 0}
	\end{align*}
	Thus, $\frac{\partial}{\partial\lambda} \bR(\hbeta_{n, \lambda}) = 0
	\implies
	\lambda = \frac{d\sigma^2}{||\B||_2^2}$
	and we conclude that
	$\Ln{n} = \frac{d\sigma^2}{||\B||_2^2}$.

	For this optimal parameter, the test risk follows from Lemma~\ref{lem:svd} as
	\begin{align}
	\label{eqn:erisk}
	&\bR(\Bn{n})
	= \bR(\hbeta_{n, \Ln{n}})\\
	&= \E_{(\gamma_1, \dots \gamma_d) \sim \Gamma_n}\left[ \sum_{i=1}^d \frac{\sigma^2}{\gamma_i^2+d\sigma^2/||\B||_2^2}  \right]
	+\sigma^2
	\end{align}
	\end{proof}

\begin{proof}[Proof of Theorem~\ref{thm:overreg}]
	We follow a similar proof strategy as in Theorem~\ref{thm:main1}:
	we invoke singular value interlacing ($\wt\gamma_i \geq \gamma_i$) for the data matrix when adding a single sample.
	We then apply Lemma~\ref{lem:svd} to argue that the test risk varies monotonically with the singular values.

	We have
	$$\bR(\hbeta_{n, \lambda})
	=
	\E_{(\gamma_1, \dots \gamma_d) \sim \Gamma}[ \sum_i
	\underbrace{\frac{||\B||_2^2 \lambda^2 / d + \sigma^2\gamma_i^2}{(\gamma_i^2+\lambda)^2}}_{S(\gamma_i)}
	]
	$$
	and we compute how each term in the sum varies with $\gamma_i$:
	\begin{align*}
	\frac{\partial}{\partial\gamma_i}
	\sum_i
	S(\gamma_i)
	&= \frac{\partial}{\partial\gamma_i}
	S(\gamma_i)\\
	&=
	(\frac{-2\gamma_i}{d})
	\frac{2 ||\B||_2^2 \lambda^2 + d\sigma^2(\gamma_i^2 - \lambda)}{(\gamma_i^2 + \lambda)^3}
	\end{align*}
	Thus we have
	\begin{align}
	\label{eqn:monotone}
	\lambda \geq \frac{d\sigma^2}{2||\B||^2}
	\implies
	\frac{\partial}{\partial\gamma_i}
	S(\gamma_i) \leq 0
	\end{align}
	By the coupling argument in Theorem~\ref{thm:main1}, this implies that the test risk is monotonic:

	\begin{align}
	&\bR(\hbeta_{n+1, \lambda})
	-
	\bR(\hbeta_{n, \lambda}) \nn\\
	&= \E_{(\wt\gamma_1, \dots \wt\gamma_d) \sim \Gamma_{n+1}}\left[ \sum_{i=1}^d S(\wt\gamma_i)  \right]
	- \E_{(\gamma_1, \dots \gamma_d) \sim \Gamma_n}\left[ \sum_{i=1}^d S(\gamma_i) \right] \nn \\
	&= \E_{
		(\{\wt\gamma_i\}, \{\gamma_i\}) \sim \Pi
	}\left[
	\sum_{i=1}^d S(\wt\gamma_i) - S(\gamma_i) \right] \\
	&\leq 0 \label{line:zero}
	\end{align}
	where $\Pi$ is the coupling.
	Line~\eqref{line:zero}
	follows from Equation~\eqref{eqn:monotone},
	and the fact that the coupling obeys $\wt\gamma_i \geq \gamma_i$.
\end{proof}

\subsection{Projection Model Proofs}
\label{sec:proj-proofs}

\begin{lemma}
\label{lem:svd-proj}
For all $\theta \in \R^p$, $d,n \in \N$, and $\lambda > 0$,
let $X \in \R^{n \x p}$ be a matrix with i.i.d. $\cN(0, 1)$
entries.
Let $P \in \R^{d \x p}$ be a random orthonormal matrix.
Define $\t{X} := XP^T$ and $\beta^* := P\theta$.

Let $(\gamma_1, \dots, \gamma_m)$ be the singular values of the data matrix $\Tilde{X}\in \R^{n \x d}$,
for $m := \max(n, d)$
(with $\gamma_i = 0$ for $i > \min(n, d)$).
Let $\Gamma_d$ be the distribution of singular values
$(\gamma_1, \dots, \gamma_m)$.

Then, the expected test risk is
\begin{align*}
&\bar{R}(\hbeta_{d, \lambda}) := \E_P \E_{\t{X}, \vec y \sim \cD^n}[R_P(\hbeta_{d, \lambda}(\Tilde{X}, \vec y)]\\
&=
\sigma^2
+
(1-\frac{d}{p})||\theta||_2^2\\
+&\E_{(\gamma_1, \dots, \gamma_m) \sim \Gamma_d
}\left[\sum_{i=1}^p
\frac{(\sigma^2 + \frac{p-d}{p}||\theta||_2^2)\gamma_i^2
+
\frac{d}{p^2}||\theta||_2^2 \lambda^2}
{(\gamma_i^2+\lambda)^2}
\right]
\end{align*}
\end{lemma}

\begin{proof}[Proof of Lemma~\ref{lem:svd-proj}]
We first define the parameter that minimizes the population
risk. It follows directly that:
\[
\beta_P^* := \argmin_{\beta \in \R^d} R_P(\beta)
= P\theta
\]

First, we can expand the risk as
\begin{align}
R(\hbeta)
&= \E_{(\t{x}, y) \sim \cD}[(\langle Px, \hbeta \rangle - y)^2]\\
&= \E_{(\t{x}, y) \sim \cD, \eta \sim \cN(0, \sigma^2)}
[(\langle x, P^T\hbeta - \theta \rangle + \eta)^2]\\
&= \sigma^2 + ||\theta - P^T \hbeta ||_{2}^2 \\
&= \sigma^2 + ||\theta - P^T \beta^* ||_{2}^2 + ||P^T \beta^* - P^T \hbeta ||_{2}^2 \\
&+ 2 \langle (\theta - P^T \beta^*),  P^T \beta^* - P^T \hbeta\rangle \label{line:cross}\\
&= \sigma^2 + ||\theta - P^T \beta^* ||_{2}^2 + ||P^T \beta^* - P^T \hbeta ||_{2}^2\\
&= \sigma^2+ ||\theta - P^T P \theta||_{2}^2 + || \beta^* -\hbeta ||_2^2
\label{line:triplet}
\end{align}

The cross terms in Line~\eqref{line:cross} vanish because the first-order optimality condition for $\beta^*$ implies that $\beta^*$ satisfies $P (\theta^* - P^T \beta^*) = 0$.
We now simplify each of the two remaining terms.

First, we have that:
\begin{align}
\E_P ||\theta - P^T P\theta ||_2^2 =
(1-\frac{d}{p})||\theta||_2^2
\label{line:ptp}
\end{align}
since $P^TP$ is an orthogonal projection onto a random $d$-dimensional subspace.

Now, recall we have
$\vec y = X\theta + \eta$
where $\eta \sim \cN(0, \sigma^2 I_n)$.
Expand this as:

\begin{align}
\vec y &= X\theta + \eta\\
&= XP^TP\theta + X(1-P^TP)\theta + \eta\\
&= \t{X}\beta^* + \eps + \eta
\end{align}
where $\eps := X(1-P^TP)\theta$.
Note that conditioned on $P$,
the three terms
$\t{X}, \eps$ and $\eta$
are conditionally independent,
since $P^TP$ and $(I-P^TP)$ project $X$ onto orthogonal subspaces.
And further,
$\eps \sim \cN(0, ||(1-P^TP)\theta||^2 I_n)$.

\begin{align}
&\E_P \E_{\t X, y} || \hbeta -\beta^* ||_2^2\\
&=
\E_P \E_{\t X, y} || (\t{X}^T \t{X} + \lambda I)^{-1}\t{X}^T y -\beta^* ||_2^2\\
&=
\E_P \E_{\t X, y, \eps, \eta}
|| (\t{X}^T \t{X} + \lambda I)^{-1}\t{X}^T ( \t{X}\beta^* + \eps + \eta) -\beta^* ||_2^2\\
&=
\E_P \E_{\t X, y, \eps, \eta}
[|| (\t{X}^T \t{X} + \lambda I)^{-1}\t{X}^T  \t{X}\beta^*  -\beta^* ||_2^2\\
&+
|| (\t{X}^T \t{X} + \lambda I)^{-1}\t{X}^T \eps ||_2^2
+
|| (\t{X}^T \t{X} + \lambda I)^{-1}\t{X}^T \eta ||_2^2
] \label{line:tbc} \\
\end{align}

Now, since $\t{X}$ is conditionally independent of $\eps$
conditioned on $P$,
\begin{align}
&\E_P \E_{\t X, y, \eps | P}
|| (\t{X}^T \t{X} + \lambda I)^{-1}\t{X}^T \eps ||_2^2\\
&=\E_P \E_{\t X | P}
[|| (\t{X}^T \t{X} + \lambda I)^{-1}\t{X}^T ||_F^2]
\E_{\eps | P} [||\eps||_2^2]\\
&=\E_{\t X}
[|| (\t{X}^T \t{X} + \lambda I)^{-1}\t{X}^T ||_F^2]
\E_{P, \eps} [||\eps||_2^2] \label{line:nocond}\\
&=\E_{\t X}
[|| (\t{X}^T \t{X} + \lambda I)^{-1}\t{X}^T ||_F^2]
\cdot \E_{P, X}[||X(1-P^TP)\theta||_2^2]
\tag{by definition of $\eps$}\\
&=\E_{\t X}
[|| (\t{X}^T \t{X} + \lambda I)^{-1}\t{X}^T ||_F^2]
(\frac{p-d}{p}||\theta||_2^2)
\label{line:eps}
\end{align}
where Line~\eqref{line:nocond} holds because the marginal
distribution of $\t{X}$ does not depend on $P$.

Similarly,
\begin{align}
&\E_P \E_{\t X, y, \eta | P}
|| (\t{X}^T \t{X} + \lambda I)^{-1}\t{X}^T \eta ||_2^2\\
&=
\sigma^2 \E_{\t X}
[|| (\t{X}^T \t{X} + \lambda I)^{-1}\t{X}^T ||_F^2]
\label{line:eta}
\end{align}

Now let $\t{X} = U \Sigma V^T$ be the full singular value decomposition of $\t{X}$,
with $U \in \R^{n \x n}, \Sigma \in \R^{n \x d},  V \in \R^{d \x d}$.
Let $(\gamma_1, \dots \gamma_m)$
denote the singular values,
where $m = \max(n, d)$ and
defining $\gamma_i = 0$ for $i > \min(n, d)$.

Observe that by symmetry, $\t{X} = XP^T$ and $P$ are independent,
because the joint distribution $(\t{X}, P)$
is equivalent to the distribution $(\t{X}Q, PQ)$ for a random orthonormal
$Q \in \R^{p \x p}$.
Thus $\t{X}$ and $\beta^* = P\theta$ are also independent,
and we have:
\begin{align}
&\E_P \E_{\t X}
|| (\t{X}^T \t{X} + \lambda I)^{-1}\t{X}^T  \t{X}\beta^*  -\beta^* ||_2^2\\
&=
\E_{\beta^*}
\E_{V, \Sigma}
[||\diag(\{\frac{-\lambda}{\gamma_i^2+\lambda}\})V^T \B ||_2^2]\\
&=
\E_{\beta^*}
\E_{V, \Sigma}
[||\diag(\{\frac{-\lambda}{\gamma_i^2+\lambda}\})V^T \B ||_2^2]\\
&=
\E_{\beta^*}
\E_{z \sim \mathrm{Unif}(||\B||_2\mathbb{S}^{d-1}), \Sigma}[||\diag(\{\frac{-\lambda}{\gamma_i^2+\lambda}\})z ||_2^2] \\
&=
\E_{\B}[
\frac{||\B||_2^2}{p} \E_{\Sigma}[\sum_i \frac{\lambda^2}{(\gamma_i^2+\lambda)^2}
]]\\
&=
\frac{1}{p}
\E_{P}[||P\theta||_2^2] \cdot
\E_{\Sigma}[\sum_i \frac{\lambda^2}{(\gamma_i^2+\lambda)^2}] \\
&=
\frac{d}{p^2}||\theta||_2^2
\E_{\Sigma}[\sum_i \frac{\lambda^2}{(\gamma_i^2+\lambda)^2}]
\label{line:beta}
\end{align}

Finally, continuing from Line~\eqref{line:tbc},
we can use Lines~\eqref{line:eps},~\eqref{line:eta}, and ~\eqref{line:beta}
to write:
\begin{align}
&\E_P \E_{\t X, y} || \hbeta -\beta^* ||_2^2\\
&=
(\sigma^2 + \frac{p-d}{p}||\theta||_2^2)
\E_{\t{X}}
[|| (\t{X}^T \t{X} + \lambda I)^{-1}\t{X}^T ||_F^2]\\
&+
\frac{d}{p^2}||\theta||_2^2
\E_{\Sigma}[\sum_i \frac{\lambda^2}{(\gamma_i^2+\lambda)^2}] \\
&=
(\sigma^2 + \frac{p-d}{p}||\theta||_2^2)
\E_{\Sigma}[\sum_i \frac{\gamma_i^2}{(\gamma_i^2+\lambda)^2}]\\
&+
\frac{d}{p^2}||\theta||_2^2
\E_{\Sigma}[\sum_i \frac{\lambda^2}{(\gamma_i^2+\lambda)^2}] \\
&=
\E_{\Sigma}[\sum_i
\frac{(\sigma^2 + \frac{p-d}{p}||\theta||_2^2)\gamma_i^2
+
\frac{d}{p^2}||\theta||_2^2 \lambda^2}
{(\gamma_i^2+\lambda)^2}
] \label{line:sum}
\end{align}

Now, we can continue from Line~\eqref{line:triplet},
and apply lines
\eqref{line:ptp}, \label{line:sum} to conclude:

\begin{align*}
\E[R(\hbeta)]
&= \sigma^2+ \E[||\theta - P^T P \theta||_{2}^2]
+ \E[|| \beta^* -\hbeta ||_2^2]\\
&=
\sigma^2
+
(1-\frac{d}{p})||\theta||_2^2\\
+&\E_{\Sigma}\left[\sum_{i=1}^p
\frac{(\sigma^2 + \frac{p-d}{p}||\theta||_2^2)\gamma_i^2
+
\frac{d}{p^2}||\theta||_2^2 \lambda^2}
{(\gamma_i^2+\lambda)^2}
\right]
\end{align*}
\end{proof}

\begin{proof}[Proof of Theorem~\ref{thm:main-proj}]
This follows analogously to the proof of Theorem~\ref{thm:main1},
Let $\t{X}_d$ and $\t{X}_{d+1}$
be the observed data matrices for $d$ and $d+1$ model size.
As in Theorem~\ref{thm:main1},
there exists a coupling $\Pi$
between the distributions $\Gamma_d$
and $\Gamma_{d+1}$ of the singular values of
$\t{X}_d$ and $\t{X}_{d+1}$ such that these singular values
are interlaced.

Thus by Lemma~\ref{lem:opt-proj},
\begin{align*}
\bar{R}(\Bn{d})
&=
\t{\sigma}^2
+
\E_{(\gamma_1, \dots, \gamma_m) \sim \Gamma_d}
\left[\sum_{i=1}^p
\frac{\t{\sigma}^2}
{
\gamma_i^2
+
\frac{\t{\sigma}^2p^2}{d||\theta||_2^2}
}
\right]\\
&\geq
\t{\sigma}^2
+
\E_{(\t\gamma_1, \dots, \t\gamma_m) \sim \Gamma_{d+1}}
\left[\sum_{i=1}^p
\frac{\t{\sigma}^2}
{
\t\gamma_i^2
+
\frac{\t{\sigma}^2p^2}{d||\theta||_2^2}
}
\right]\\
&=\bar{R}(\Bn{d+1})
\end{align*}
\end{proof}

\subsection{Nonisotropic Reduction}
\label{sec:reduction}
Here we observe that results on isotropic regression in Section~\ref{sec:sample}
also imply that ridge regression can be made sample-monotonic
even for non-isotropic covariates, if an appropriate regularzier is applied.
Specifically, the regularizer depends on the covariance on the inputs.
This follows from a general equivalence between the non-isotropic and isotropic problems.

\begin{lemma}
\label{lem:reduction}
For all $n \in \N, d \in \N, \lambda \in \R, \sigma \in \R$,
covariance $\Sigma \in \R^{d \x d}$,
PSD matrix $M \in \R^{d \x d}$,
and
ground-truth $\beta^* \in \R^d$, the following holds.

Consider the following two problems:
\begin{enumerate}
    \item Regularized regression on isotropic covariates, and an $M$-regularizer.
    That is, suppose $n$ samples $(x, y)$ are drawn with covariates
    $x \sim \cN(0, I_d)$ and response $y = \langle \beta^*, x \rangle + \cN(0, \sigma^2)$.
    Let $X \in \R^{n \x d}$ be the matrix of covariates, and $\vec{y}$ the vector of responses.
    Consider
    \begin{align}
        \hbeta_\lambda := \argmin_\beta ||X \beta - \vec{y} ||_2^2 + \lambda ||\beta||_{M}^2
    \end{align}
    Let $\bR = \E_{X, y}[||\hbeta - \beta^*||_2^2] + \sigma^2$
    be the expected test risk of the above estimator.
    
    \item Regularized regression with covariance $\Sigma$,
    and an $(\Sigma^{1/2}M\Sigma^{1/2})$-regularizer.
    That is, suppose $n$ samples $(\t{x}, y)$ are drawn with covariates
    $\t{x} \sim \cN(0, \Sigma)$ and response
    $y = \langle z^*, \t{x} \rangle + \cN(0, \sigma^2)$,
    for
    $$z^*=\Sigma^{-1/2}\beta^*$$
    Let $\t{X} \in \R^{n \x d}$ be the matrix of covariates, and $\vec{y}$ the vector of responses.
    Consider
    \begin{align}
        \hat{z}_\lambda := \argmin_z
        ||\t{X} z - \vec{y} ||_2^2 + \lambda ||z||^2_{\Sigma^{1/2}M\Sigma^{1/2}}
    \end{align}
    Let $\t{R} = \E_{\t{X}, y}[||\hat{z} - z^*||_\Sigma^2] + \sigma^2$
    be the expected test risk of the above estimator.
\end{enumerate}
Then, the expected test risks of the above two problems are identical:
$$
\bR = \t{R}
$$
\end{lemma}
\begin{proof}[Proof of Lemma~\ref{lem:reduction}]
The distribution of $\t{X}$ in the Problem 2
is equivalent to $X \Sigma^{1/2}$, where $X$ is as in Problem 1.
Thus, the two settings are equivalent by the change-of-variable $\beta = \Sigma^{1/2} z$.
Specifically,
\begin{align}
&\hat{z}_\lambda := \argmin_z
||\t{X} z - \vec{y} ||_2^2 + \lambda ||z||^2_{\Sigma^{1/2}M\Sigma^{1/2}}\\
&= \argmin_z
||X \Sigma^{1/2} z - \vec{y} ||_2^2 + \lambda z^T \Sigma^{1/2} M \Sigma^{1/2}z \\
&= \argmin_z
||X \Sigma^{1/2} z - \vec{y} ||_2^2 + \lambda z^T \Sigma^{1/2} M \Sigma^{1/2}z \\
&= \Sigma^{-1/2} \argmin_{\beta = \Sigma^{1/2}z }
||X \beta - \vec{y} ||_2^2 + \lambda \beta^T M \beta
\end{align}
Further, the response $\langle z^*, \t{x} \rangle = \langle \beta, x \rangle$,
and the test risk transforms identically:
\begin{align}
\t{R}
&= \E_{\t{X}, y}[||\hat{z} - z^*||_\Sigma^2] + \sigma^2\\
&= \E_{X, y}[||\hat{\beta} - \beta^*||^2_2] + \sigma^2\\
&=\bR
\end{align}
\end{proof}

This implies that if the covariance $\Sigma$ is known,
then ridge regression with a $\Sigma^{-1}$ regularizer is sample-monotonic.

\begin{theorem}
For all $n \in \N, d \in \N, \sigma \in \R$,
covariance $\Sigma \in \R^{d \x d}$,
and ground-truths $\beta^* \in \R^d$, the following holds.

Suppose $n$ samples $(x, y)$ are drawn with covariates
$x \sim \cN(0, \Sigma)$ and response
$y = \langle \beta^*, x \rangle + \cN(0, \sigma^2)$.
Let $X \in \R^{n \x d}$ be the matrix of covariates, and $\vec{y}$ the vector of responses.
For $\lambda > 0$, consider the ridge regression estimator with $\Sigma^{-1}$-regularizer:
\begin{align}
    \hbeta_{n, \lambda} := \argmin_\beta
    ||X \beta - \vec{y} ||_2^2 + \lambda ||\beta||^2_{\Sigma^{-1}}
\end{align}

Let $\bR(\hbeta_{n, \lambda}) := \E_{\hbeta} ||\hbeta - \beta^*||_\Sigma + \sigma^2$
be the expected test risk of the above estimator.
Let $\Ln{n}$ be the optimal ridge parameter (that achieves the minimum expected risk) given $n$ samples:
\begin{align}
\Ln{n} &:= \argmin_\lambda \bR(\hbeta_{n, \lambda})) 
\end{align}
And let $\Bn{n}$ be the estimator that corresponds to the $\Ln{n}$.
Then, 
the expected test risk of optimally-regularized linear regression is monotonic in samples:
$$
\bR(\Bn{n+1})
\leq 
\bR(\Bn{n})
$$
\end{theorem}
\begin{proof}
This follows directly by applying the reduction in Lemma~\ref{lem:reduction}
for $M=I_d$ to reduce to the isotropic case,
and then applying the monotonicity of isotropic regression
from Theorem~\ref{thm:main1}.
\end{proof}

\subsubsection{Monotonicity Conjecture}
\label{sec:monoconj}
\begin{lemma}
\label{lem:monoconj}
Conjecture~\ref{conj:numeric}
implies Conjecture~\ref{conj:main}.
\end{lemma}
\begin{proof}
By the reduction in Section~\ref{sec:reduction},
showing monotonicity for
non-isotropic regression with an isotropic regularizer
is equivalent to showing monotonicity
for isotropic regression with a non-isotropic regularizer.
Thus, we consider the latter.
Specifically, Conjecture~\ref{conj:main}
is equivalent to showing monotonicity for
the estimator
\begin{align}
    \hbeta_{n, \lambda}
    &:= \argmin_\beta
    ||X \beta - \vec{y} ||_2^2 + \lambda ||\beta||^2_{\Sigma^{-1}}\\
    &= (X^TX + \lambda \Sigma^{-1})^{-1}X^T y
\end{align}
where $x \sim \cN(0, I)$ is isotropic,
and $y \sim \langle x, \beta^* \rangle + \cN(0, \sigma^2)$.

Now, letting $Q := \Sigma^{-1}$, the expected test risk of this estimator
for $n$ samples is:
\begin{align*}
    \bR(\hbeta_{n, \lambda})
    &= \E_{X, y}[|| \hbeta_{n, \lambda} - \B ||_2^2] + \sigma^2 \\
    &= \E_{X, y}[|| (X^TX + \lambda Q)^{-1}X^T y - \B ||_2^2] +
    \sigma^2\\
    &= \E_{X, \eta \sim \cN(0, \sigma^2 I_n)}
    [|| (X^TX + \lambda Q)^{-1}X^T(X\B + \eta) - \B ||_2^2] +
    \sigma^2\\
    &=
    \E_{X}[ || (X^TX + \lambda Q)^{-1}X^TX\B - \B) ||_2^2]
    +
    \sigma^2 \E_{X}[ || (X^TX + \lambda Q)^{-1}X^T||_F^2 ]
    + \sigma^2\\
    &=
    \E_{X}[ || (X^TX + \lambda Q)^{-1}(X^TX + \lambda Q - \lambda Q)\B - \B) ||_2^2]
    +
    \sigma^2 \E_{X}[ || (X^TX + \lambda Q)^{-1}X^T||_F^2 ]
    + \sigma^2\\
    &=
    \lambda^2 \E_{X}[ ||(X^TX+\lambda Q)^{-1}Q \B ||_2^2]
    +
    \sigma^2 \E_{X}[ || (X^TX + \lambda Q)^{-1}X^T||_F^2 ]
    + \sigma^2 \\
    &=
    (\B)^T G_\lambda^n \B
    +
    \sigma^2 H_\lambda^n
    + \sigma^2
\end{align*}

Consider the infimum
\begin{align}
\label{eqn:inf}
\inf_{\lambda \geq 0 } \bR(\hbeta_{n, \lambda})
\end{align}
We consider several cases below.

{\bf Case (1).} Suppose the infimum in Equation~\ref{eqn:inf}
is achieved in the limit $\lambda \to +\infty$.
In this case, monotonicity trivially holds, since
$$
\lim_{\lambda \to \infty}
\bR(\hbeta_{n, \lambda})
=
\bR(\vec{0})
=
\lim_{\lambda \to \infty}
\bR(\hbeta_{n+1, \lambda})
$$

{\bf Case (2).} Suppose the infimum in Equation~\ref{eqn:inf}
is achieved by some $\lambda = \Ln{n}$ in the interior of the set $(0, \infty)$.

Because $\bR(\hbeta_{n, \lambda})$
is continuous and differentiable in $\lambda$ for all
$\lambda \in (0, \infty)$,
we have that $\Ln{n}$
must satisfy the following first-order optimality condition:
\begin{align}
    \frac{d \bR(\hbeta_{n, \lambda})}{d \lambda} \Big \vert_{\lambda = \Ln{n}}&= 0 \\
\implies (\B)^T \frac{dG_\lambda^n}{d\lambda} \B
    +
    \sigma^2 \frac{dH_\lambda^n}{d\lambda} \Big \vert_{\lambda = \Ln{n}}
    &= 0
    \label{eqn:first1}
\end{align}
We will later use this condition to show monotonicity.

{\bf Case (3).} Suppose the infimum in Equation~\ref{eqn:inf}
is achieved at $\Ln{n} = 0$.
Recall, we define $\hbeta_{n, 0}:= \lim_{\lambda \rightarrow 0+} \hbeta_{n, \lambda}$.
This means that,
\begin{align}
\frac{d \bR(\hbeta_{n, \lambda})}{d \lambda} \Big \vert_{\lambda = 0}
=
(\B)^T \frac{dG_\lambda^n}{d\lambda} \B
+
\sigma^2 \frac{dH_\lambda^n}{d\lambda} \Big \vert_{\lambda = 0}
&\ge 0
\label{eqn:first2}
\end{align}

Note that since $\frac{dH_\lambda^n}{d\lambda} \le 0$,
both Equations~\eqref{eqn:first1} and~\eqref{eqn:first2}
in Case (2) and Case (3) respectively imply that
\begin{align}
\sigma^2
&\le
- \frac{(\B)^T (\frac{dG_\lambda^n}{d\lambda}) \B}
{dH_\lambda^n/d\lambda}
\Big \vert_{\lambda = \Ln{n}}
\label{eqn:const}
\end{align}

Now, assuming Conjecture~\ref{conj:numeric},
we will show that the choice of $\Ln{n}$ in Cases (2) and (3)
has non-increasing test risk for
$(n+1)$ samples.
That is,
$$
\bR(\hbeta_{n, \Ln{n}})
\geq
\bR(\hbeta_{n+1, \Ln{n}})
$$
This implies the desired monotonicity, since
$\bR(\hbeta_{n+1, \Ln{n}}) \geq
\bR(\hbeta_{n+1, \Ln{n+1}})$.

We first consider the case when
$H_{\lambda}^n - H_\lambda^{n+1} \vert_{\lambda = \Ln{n}} \ge 0$.
In this case, because
$G_\lambda^n  - G_\lambda^{n+1} \succeq 0$ by assumption, we have
\begin{align}
\bR(\hbeta_{n, \Ln{n}})
-
\bR(\hbeta_{n+1, \Ln{n}})
&=
(\B)^T (G_\lambda^n  - G_\lambda^{n+1}) \B
+
\sigma^2 (H_\lambda^n - H_\lambda^{n+1})
\Big \vert_{\lambda = \Ln{n}}
\\
&
\geq 0
\end{align}
Otherwise, assume.
$H_{\lambda}^n - H_\lambda^{n+1} \vert_{\lambda = \Ln{n}} \leq 0$.
Then we have:
\begin{align}
\bR(\hbeta_{n, \Ln{n}})
-
\bR(\hbeta_{n+1, \Ln{n}})
&=
    (\B)^T (G_\lambda^n  - G_\lambda^{n+1}) \B
+
    \sigma^2 (H_\lambda^n - H_\lambda^{n+1})
\Big \vert_{\lambda = \Ln{n}}
    \\
&\ge
    (\B)^T (G_\lambda^n  - G_\lambda^{n+1}) \B
    - (\B)^T (\frac{d G_\lambda^n}{d\lambda}) \B
    \frac{(H_\lambda^n - H_\lambda^{n+1})}{dH_\lambda^n/d\lambda}
\Big \vert_{\lambda = \Ln{n}}
    \tag{by Equation~\eqref{eqn:const}, and $H_\lambda^n - H_\lambda^{n+1}\le 0$}
    \\
&=
    (\B)^T \underbrace{\left(
(G_\lambda^{n} - G_\lambda^{n+1})
-
(H_\lambda^n - H_\lambda^{n+1})
\frac{
d G_\lambda^n/d\lambda
}
{
d H_\lambda^n/d\lambda
}
    \right)
\Big \vert_{\lambda = \Ln{n}}
    }_{\succeq 0
    \text{ by Conjecture~\ref{conj:numeric}}}
    \B\\
\geq 0
\end{align}
as desired.

\end{proof}

\subsection{Additional Plots}
\label{sec:plots}

\begin{figure}[h]
    \centering
    \includegraphics[width=0.7\textwidth]{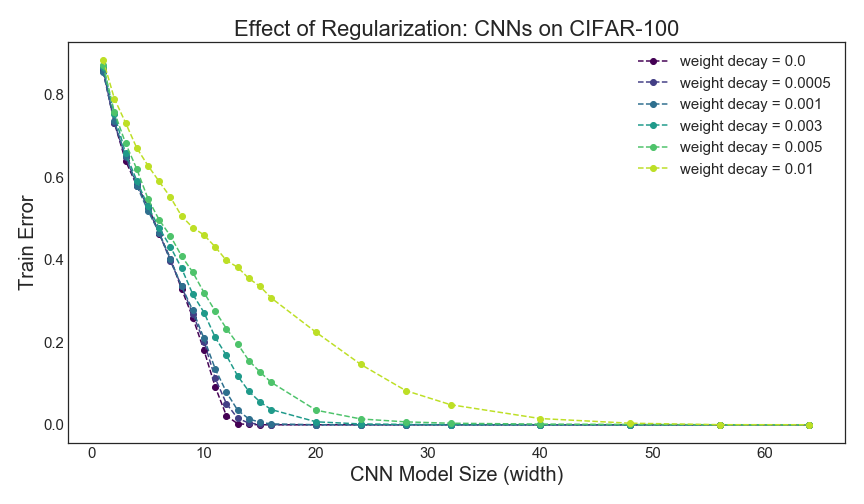}
    \caption{Train Error vs. Model Size for 5-layer CNNs on CIFAR-100,
    with $\ell_2$ regularization (weight decay).}
    \label{fig:cifar-train}
\end{figure}

\begin{figure}[h]
    \centering
    \includegraphics[width=0.55\textwidth]{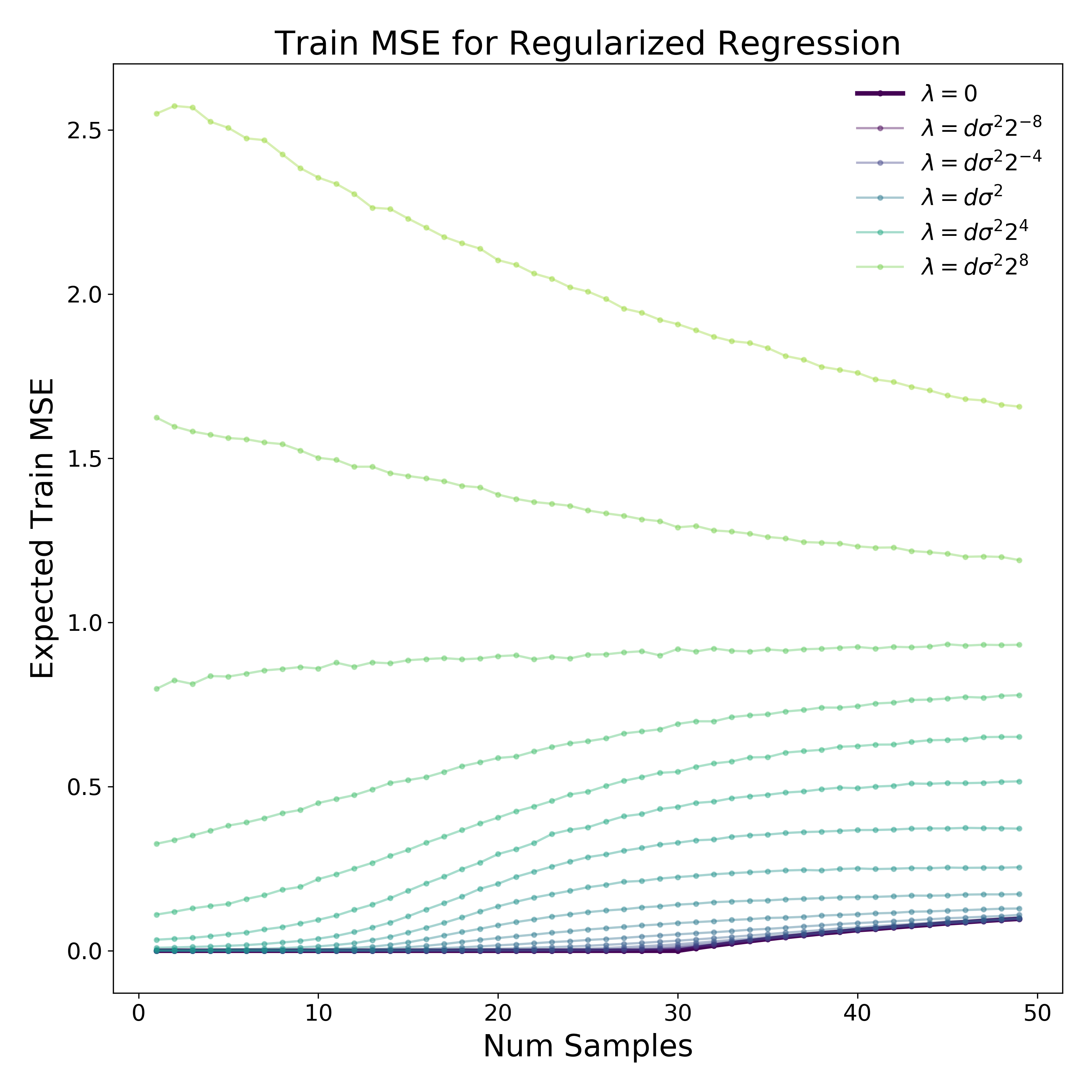}
    \caption{Train MSE vs. Num. Samples for Non-Isotropic Ridge Regression in $d=30$ dimensions, in the setting of Figure~\ref{fig:nonisotropic}.
    Plotting train MSE: $\frac{1}{n}||X \hbeta - \vec{y}||_2^2$.
    }
    \label{fig:nonisotropic-train}
\end{figure}

\begin{figure}[h]
    \centering
    \includegraphics[width=0.55\textwidth]{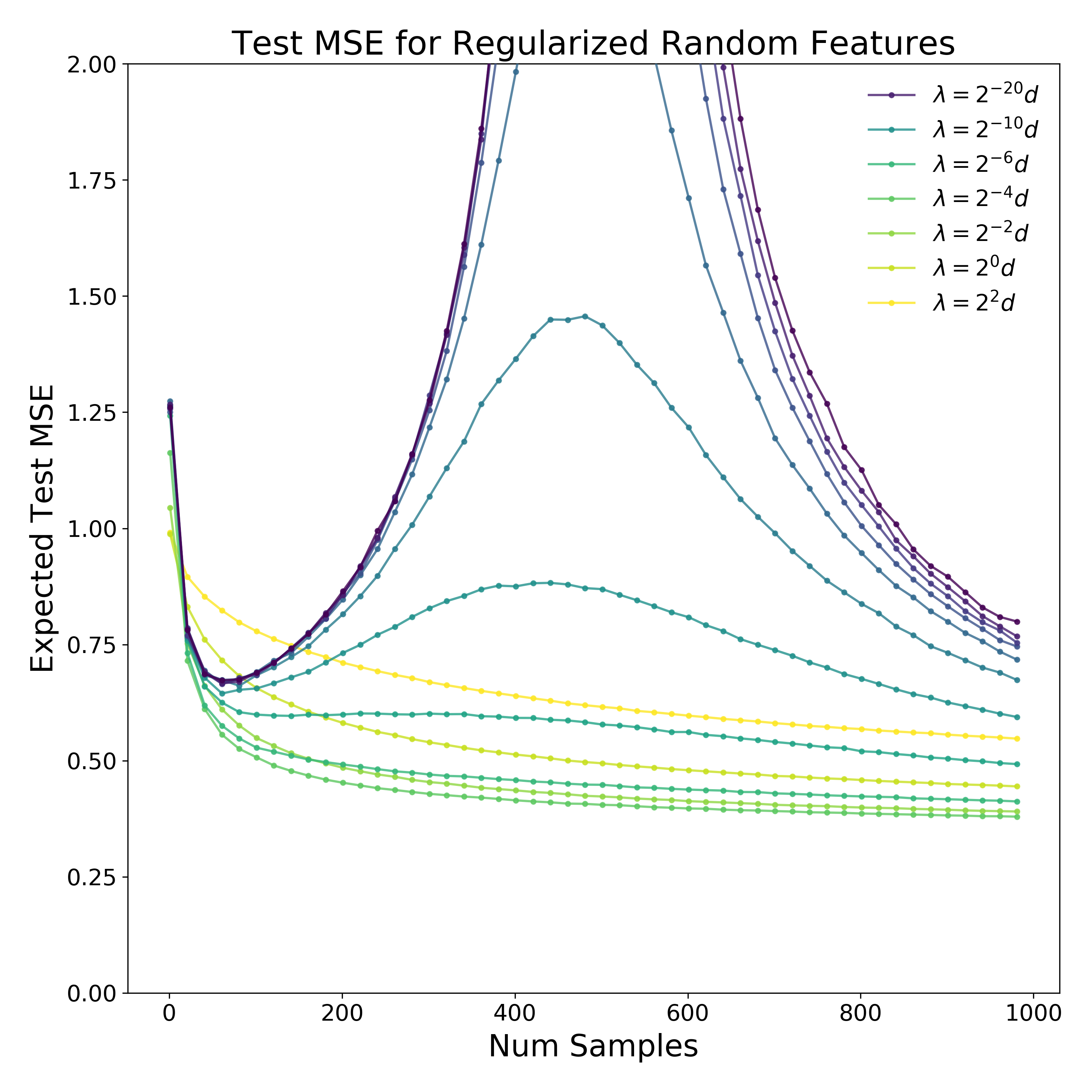}
    \caption{Test Mean Squared Error vs. Num Train Samples
    for Random ReLU Features on Fashion-MNIST, with $D=500$ features.}
    \label{fig:relu-samps-mse}
\end{figure}

\begin{figure}[h]
    \centering
    \includegraphics[width=0.55\textwidth]{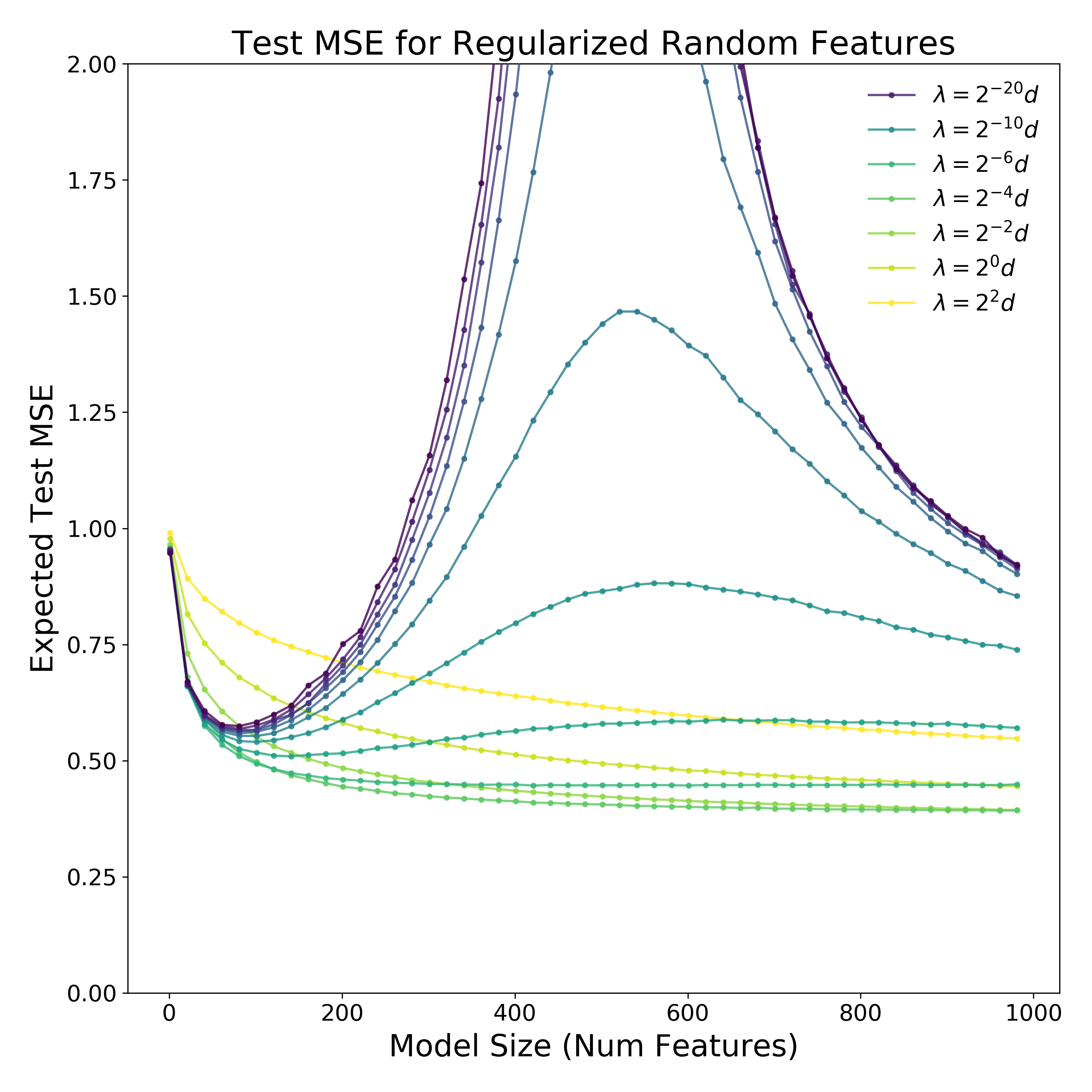}
    \caption{Test Mean Squared Error vs. Num Features
    for Random ReLU Features on Fashion-MNIST, with $n=500$ samples.}
    \label{fig:relu-model-mse}
\end{figure}

\end{document}